\documentclass{article}

\usepackage[numbers,sort&compress]{natbib}
 \usepackage[main, final]{neurips_2026}

\everydisplay{\small}

\usepackage{wrapfig}

\usepackage{multirow,mathtools }

\usepackage{pifont}
\usepackage{color, colortbl}

\usepackage{blindtext}
\usepackage{lipsum}

\usepackage{multirow}
\usepackage{listings}

\usepackage{bbm}

\usepackage [english]{babel}
\usepackage[autostyle, english = american]{csquotes}

\usepackage{pifont}
\usepackage{url}
\usepackage[most]{tcolorbox}

\usepackage{lipsum}
\usepackage{soul}
\usepackage{xcolor}
\usepackage{wrapfig}
\usepackage{multirow,mathtools } 

\usepackage{adjustbox}
\MakeOuterQuote{"}

\usepackage{microtype}
\usepackage{graphicx}
\usepackage{booktabs} 

\usepackage{hyperref}

\usepackage{amsmath}

\usepackage[capitalize,noabbrev]{cleveref}

\usepackage{nicefrac}       

\usepackage{tablefootnote}

\usepackage{float} 

\usepackage{amsmath,amsfonts,bm}

\def\eqref#1{(\ref{#1})}

\def\1{\bm{1}}

\DeclareMathAlphabet{\mathsfit}{\encodingdefault}{\sfdefault}{m}{sl}
\SetMathAlphabet{\mathsfit}{bold}{\encodingdefault}{\sfdefault}{bx}{n}

\DeclareMathOperator*{\minimize}{\text{minimize}}

\newtheorem{myprop}{\bf{Proposition}}

\newtheorem{mythr}{\bf{Theorem}}
\newtheorem{mylemma}{\bf{Lemma}}

\newcommand{\Def}[0]{\mathrel{\mathop:}=}

\usepackage{pifont}

\usepackage{color, colortbl}
\definecolor{Gray}{gray}{0.93}
\definecolor{Orange}{rgb}{1,0.5,0}
\definecolor{DGray}{gray}{0.83}
\definecolor{LightCyan}{rgb}{0.88,1,1}

\definecolor{WarnREd}{rgb}{1,0.4,0.4}
\definecolor{WarnOrange}{rgb}{1,0.682,0.502}
\definecolor{WarnPink}{rgb}{0.9176, 0.7215, 0.7215}
\definecolor{GoodGreen}{rgb}{0.5019, 0.9215, 0.6039}

\usepackage[T1]{fontenc}

\definecolor{styleblue}{HTML}{504099}
\definecolor{mypurple}{HTML}{9391ff}

\definecolor{bluegray}{rgb}{0.4, 0.6, 0.8}
\definecolor{ceruleanblue}{rgb}{0.16, 0.32, 0.75}

\usepackage{listings}
\usepackage{xcolor}
\usepackage{float}      
\usepackage{caption}    

\lstdefinestyle{simple}{
  language=Python,
  basicstyle=\ttfamily\scriptsize, 
  numbers=none,                    
  frame=single,
  rulecolor=\color{black},
  backgroundcolor=\color[gray]{0.96},
  showstringspaces=false,
  keepspaces=true,
  columns=fullflexible,
  breaklines=true,
  breakatwhitespace=true,
  tabsize=2
}

\usepackage[utf8]{inputenc} 
\usepackage[T1]{fontenc}    
\usepackage{hyperref}       
\hypersetup{hidelinks}
\usepackage{url}            
\usepackage{booktabs}       
\usepackage{amsfonts}       
\usepackage{nicefrac}       
\usepackage{microtype}      
\usepackage{xcolor}         

\usepackage{amsmath}     
\usepackage{algorithm}    
\usepackage{algorithmic}
\usepackage{url}
\usepackage{array}      
\usepackage{booktabs}   
\usepackage{multirow}
\usepackage{enumitem}
\usepackage{amsmath}

\usepackage[table]{xcolor} 

\usepackage{makecell}
\usepackage[most]{tcolorbox}
\usepackage{graphicx}
\usepackage{subcaption} 
\usepackage{wrapfig}
\definecolor{SL_color}{rgb}{0.858, 0.188, 0.478}

\definecolor{lightblue}{rgb}{0.88, 0.95, 1.0}
\definecolor{lightorangepink}{rgb}{1.0, 0.85, 0.8}
\definecolor{lightgreen}{rgb}{0.88, 1, 0.88}
\newcommand{\orangedot}{\textcolor{orange}{\Large\bullet}\xspace}

\definecolor{deepviolet}{RGB}{90,40,120}

\usepackage{threeparttable}

\usepackage{microtype}

\usepackage{tikz}

\title{Powering Up Zeroth-Order Training via Subspace Gradient Orthogonalization}

\author{%
\textbf{Yicheng Lang}$^{1}$\quad \textbf{Changsheng Wang}$^{1}$  \quad \textbf{Yihua Zhang}$^{1}$ \quad \textbf{Mingyi Hong} $^{2}$ \\
\quad \textbf{Zheng Zhang} $^{3}$ \quad \textbf{Wotao Yin} $^{4}$ \quad \textbf{Sijia Liu}$^{1,5}$\\
$^{1}$Michigan State University~~ 
$^{2}$University of Minnesota~~ 
$^{3}$University of California, Santa Barbara\\
$^{4}$DAMO
Academy, Alibaba Group US ~~
$^{5}$IBM Research, USA
}

\begin{document}

\maketitle

\vspace*{-8mm}
\begin{abstract}
\vspace*{-2mm}
Zeroth-order (ZO) optimization provides a gradient-free alternative to first-order (FO) methods by estimating gradients via finite differences of function evaluations, and has recently emerged as a memory-efficient paradigm for fine-tuning large-scale models by avoiding backpropagation.
However, ZO optimization 
has a fundamental tension between accuracy and query efficiency.
In this work, we show that this can be substantially improved by unifying two complementary principles: (i) a projection-based subspace view that reduces gradient estimation variance by exploiting the intrinsic low-rank structure of model updates, and (ii) Muon-style spectral optimization that applies gradient orthogonalization to extract informative spectral structure from noisy ZO gradients. These form a unified framework of subspace gradient orthogonalization, which we instantiate in a new method, \textbf{ZO-Muon}, admitting a natural interpretation as a low-rank Muon optimizer in the ZO setting.
Extensive experiments on large language models (LLMs) and vision transformers (ViTs) demonstrate that ZO-Muon significantly accelerates convergence and achieves a win–win improvement in accuracy and query/runtime efficiency. Notably, compared to the popular MeZO baseline, ZO-Muon requires only 24.7\% of the queries to reach the same SST-2 performance for LLM fine-tuning, and improves accuracy by 25.1\% on ViT-L fine-tuning on CIFAR-100. Code is available at \url{https://github.com/OPTML-Group/ZO-Muon}.
\end{abstract}

\section{Introduction}
\label{sec: intro}
\vspace*{-1.0em} 
Training deep models such as LLMs \citep{achiam2023gpt,zhang2022opt,touvron2023llama} relies on backpropagation (BP) to compute first-order (\textbf{FO}) gradients for optimization \citep{amari1993backpropagation,kingma2014adam}. However, as model sizes scale, the GPU memory cost of BP grows substantially \citep{chen2025a}. This motivates a fundamental question with practical impact: Can deep models be trained in a \emph{gradient-free} manner without relying on BP?
Toward this end, zeroth-order (\textbf{ZO}) optimization \citep{liu2020primer,duchi2015optimal,nesterov2017random} offers a theoretically grounded alternative to FO optimization by estimating gradients via finite differences of function values using only forward passes of the model, thereby substantially reducing the memory footprint during training \citep{malladi2023finetuning,zhang2024revisiting,tan2025harmony,liu2025sparse}.
Although ZO optimization also arises in black-box or gray-box application settings \citep{chen2017zoo,ilyas2018black,zhao2019design,chen2023deepzero} (where FO gradients are infeasible or inaccessible), \emph{memory-efficient fine-tuning of large models} has recently become one of its most impactful and arguably primary applications, as first proposed in~\citep{malladi2023finetuning} and a focus of this work.

\begin{figure*}[htb]
    \centering
    \begin{tabular}{cc}
      \hspace*{-4mm}\begin{tabular}{cc}
          \includegraphics[width=0.24\linewidth]{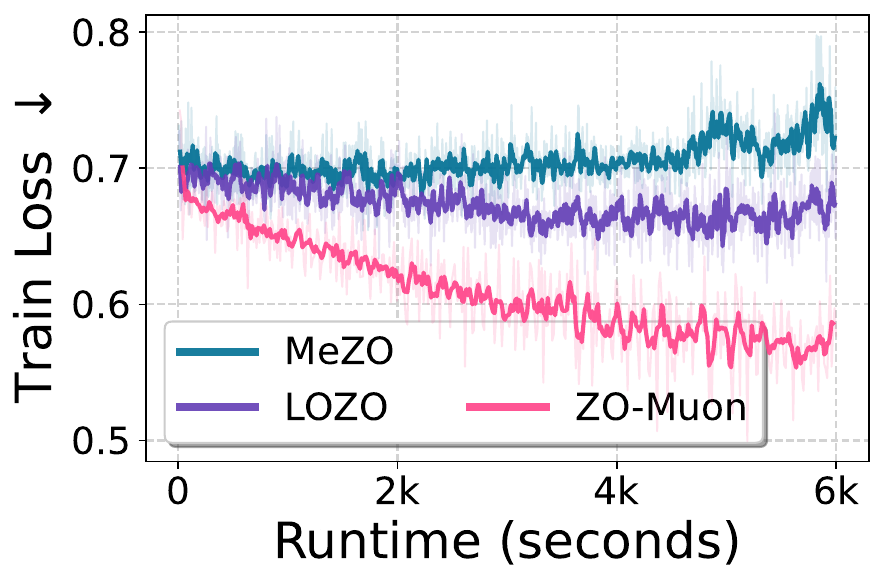} &   \hspace*{-4mm}
        \includegraphics[width=0.24\linewidth]{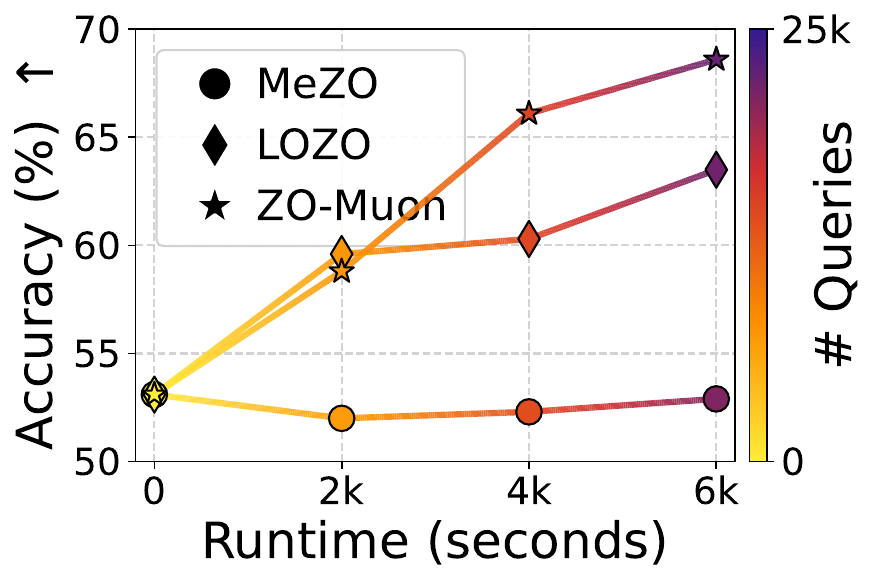}
      \end{tabular}   &  
       \hspace*{-7mm}      \begin{tabular}{cc}
          \includegraphics[width=0.24\linewidth]{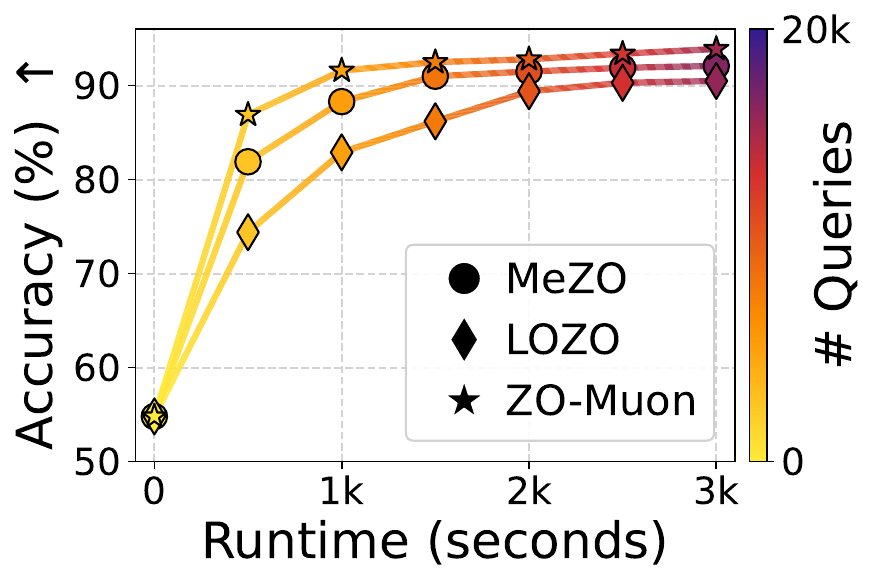} & 
           \hspace*{-4mm}
        \includegraphics[width=0.24\linewidth]{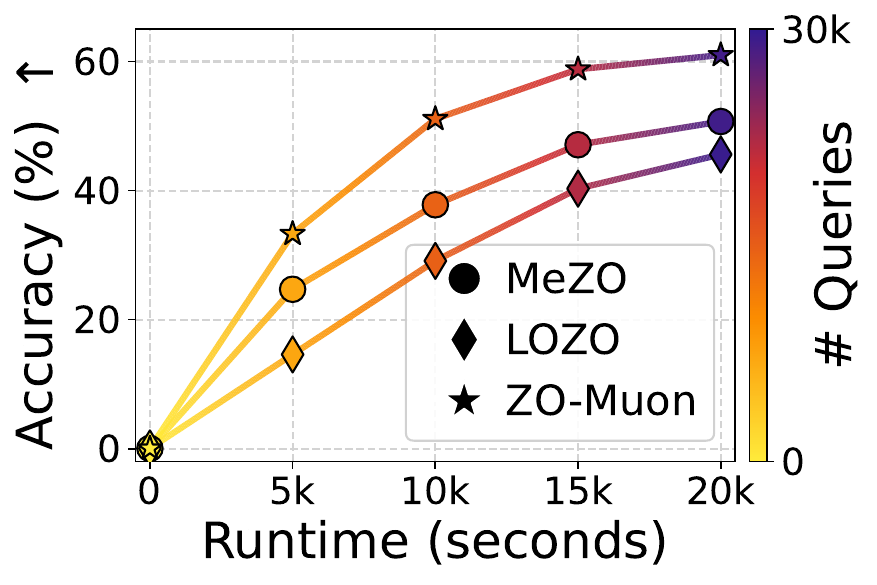} 
      \end{tabular} 
      \\
       \hspace*{-3mm} \footnotesize{(a) \textit{LLM fine-tuning}: OPT-1.3B on RTE}    &    \hspace*{-7mm} \footnotesize{(b) \textit{ViT fine-tuning}: ViT-B on CIFAR-10 (left) \& 100 (right)}
    \end{tabular}
    \vspace*{-3mm}
    \caption{\small{
 ZO-Muon achieves faster convergence and better fine-tuning accuracy than the SOTA ZO baselines MeZO \citep{malladi2023finetuning} and LOZO \citep{chen2025enhancing} against runtime. 
 \textit{(a)} OPT-1.3B fine-tuned on the RTE task: training loss (\textit{left}) and test accuracy (\textit{right}) vs. runtime, with the cumulative query count indicated by the color bar in the right subplot. 
\textit{(b)} ViT-B fine-tuned on CIFAR-10 (\textit{left}) and CIFAR-100 (\textit{right}), shown in the same format as in (a, right).
}}
\label{fig:motivating-figure}
\end{figure*}

\begin{figure*}[htb]
    \centering
    \vspace*{-1.8em}
    
    \begin{tabular}{cc}
       \hspace*{-3mm}  \includegraphics[width=0.48\linewidth]{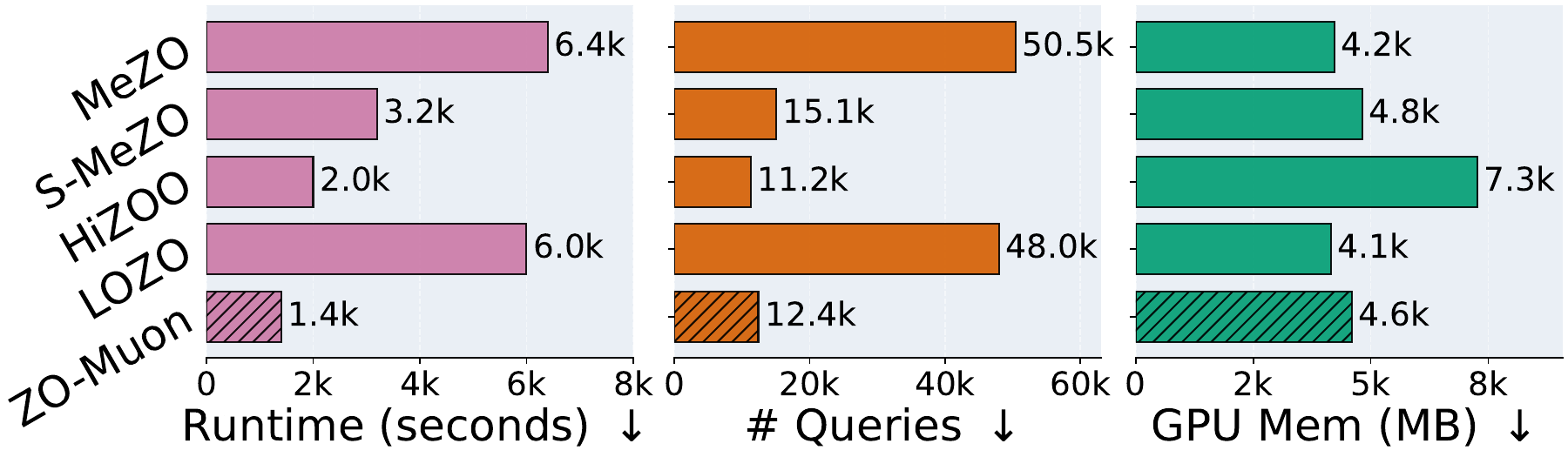} & 
         \hspace*{-3mm}  \includegraphics[width=0.48\linewidth]{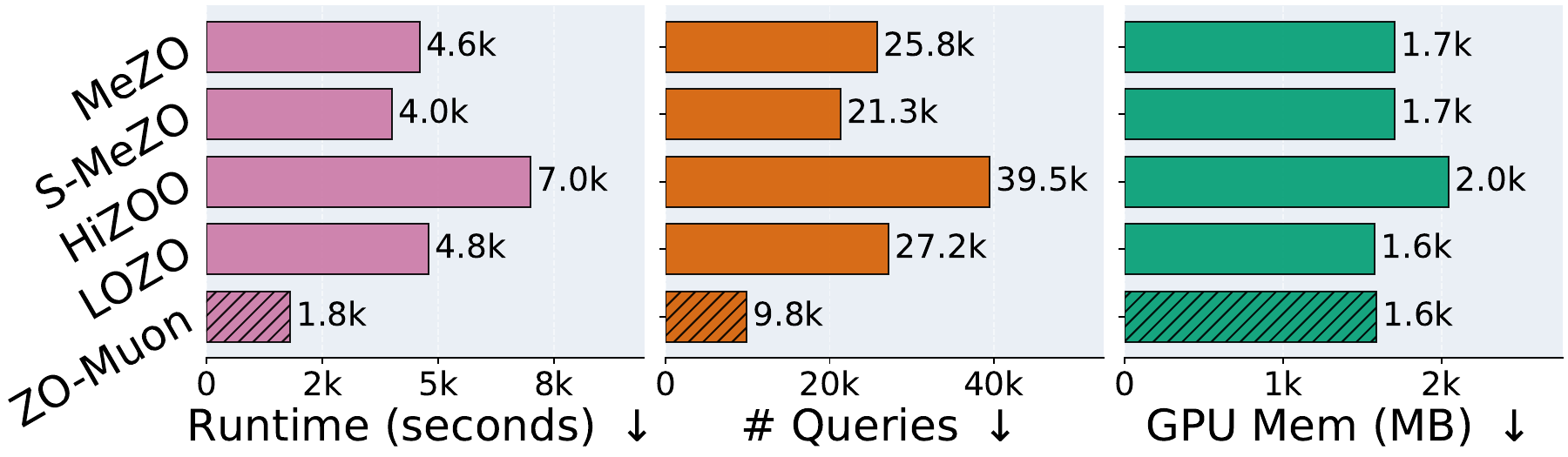} \\
        \vspace*{-1mm}
        \hspace*{-3mm}
        \footnotesize{(a) Efficiency comparison for OPT-1.3B fine-tuned on SST-2} &  \hspace*{-3mm}
        \footnotesize{(b) Efficiency comparison for ViT-B fine-tuned on CIFAR-10} \\
    \end{tabular}
    \vspace*{-2mm}
    \caption{\small{
ZO-Muon yields lower runtime and fewer function queries than all baselines (under comparable GPU memory usage), when reaching the target fine-tuning accuracies: \textit{(a)} $0.90$ on SST-2 and \textit{(b)} $0.93$ on CIFAR-10. Baselines include MeZO \citep{malladi2023finetuning}, SparseMeZO (S-MeZO) \citep{liu2025sparse}, HiZOO \citep{zhao2025secondorder}, and LOZO \citep{chen2025enhancing}.}
    }
    \vspace*{-10mm}    
    \label{fig:runtime-gpu-compare}
\end{figure*}

Despite the promise of ZO optimization, achieving convergence and accuracy comparable to FO optimization remains highly nontrivial. A key challenge of ZO optimization is the high variance of query-based gradient estimation, which becomes increasingly severe as the problem dimension grows \citep{duchi2015optimal,nesterov2017random,liu2018zeroth,liu2018zeroth_SVRG,chen2023deepzero}.
Therefore, to reduce the variance of ZO gradient estimation, existing work has focused on designing improved gradient estimators. Examples include preconditioning the estimator with Hessian information \citep{zhao2025pazo,zhao2025secondorder}, optimizing perturbation directions before computing directional derivatives \citep{ma2025revisiting}, 
and imposing structural priors such as sparsity \citep{guo2024zeroth,liu2025sparse} or low-rankness \citep{chen2025enhancing,yu2025zeroth} in the gradient estimation.
However, these designs still primarily operate in the original parameter space of the model, which limits the extent of variance reduction. As a result, a significant performance gap still remains between ZO and FO methods in deep model training.

Another principled approach to reducing variance in ZO optimization is to use multiple queries (at the cost of increased query complexity) to average out estimation noise. However, a linear increase in the number of queries (or runtime) yields only a \emph{sub-linear} improvement in performance (and may even lead to a {multi-query paradox} with no improvement) \citep{lin2025multi}, making it challenging to strike a favorable \textit{accuracy-efficiency trade-off}.
For example, FZOO \citep{dang2025fzoo} requires a substantial number of queries to estimate the variance of function evaluations for gradient normalization. Similarly, SVRG-based ZO methods \citep{liu2018zeroth_SVRG,ji2019improved} also require an excessive number of queries to construct error-correction terms for determining descent directions from estimated gradients.
ZO-Align \citep{lin2025multi} leverages projection alignment across multiple queries, but is difficult to scale because it must store multiple ZO gradients in large matrices.
MUZO \citep{peng2025muzo} takes not only multiple queries but also the more complex Adam optimizer in ZO optimization, which incurs additional memory overhead.

As mentioned above, although ZO optimization for model training has made substantial recent progress, its core bottlenecks, limited accuracy due to high variance and a poor accuracy-efficiency trade-off, remain unresolved. Thus, we ask:
\vspace*{-4mm} 
\begin{center}
       \textbf{(Q)} 
       \textit{Can we advance ZO training to achieve a win-win in both accuracy and efficiency?}
\end{center}
\vspace*{-4mm} 
To address \textbf{(Q)}, we propose a novel \emph{subspace gradient orthogonalization} viewpoint to revisit and redesign ZO optimization. This framework consists of two key components: \emph{projection-based subspace ZO optimization} and \emph{gradient-orthogonalized spectral optimization}. The former reduces the variance of ZO gradient estimation by projecting the original full-space problem into a lower-dimensional subspace. The latter further improves optimization accuracy by leveraging gradient orthogonalization in the recent Muon optimizer \citep{jordan2024muon,liu2025muon}, which can extract meaningful descent directions from noisy ZO gradient estimates rather than directly using their entry-wise values.
We term the proposed subspace gradient orthogonalization-based ZO optimization approach ``\textbf{ZO-Muon}'' to echo its connection to Muon.
As shown in \textbf{Fig.\,\ref{fig:motivating-figure}} and \textbf{Fig.\,\ref{fig:runtime-gpu-compare}}, ZO-Muon consistently outperforms a range of state-of-the-art (SOTA) baselines in both convergence speed and accuracy. Moreover, these performance gains do not incur additional query cost under the same runtime budget, indicating improved efficiency in both runtime and query usage across LLM and ViT fine-tuning. This achieves the desired win-win outcome posed in \textbf{(Q)}.
 We summarize our main \textbf{contributions} as follows.

\ding{172} (Formulation)
We revisit ZO optimization from a projection-based perspective and propose a \emph{subspace randomized gradient estimator} (Subspace RGE), which substantially reduces the variance of gradient estimation by operating in a low-dimensional subspace.

\ding{173} {(Algorithm)} We show that Muon’s \emph{gradient orthogonalization} provides complementary gains when meeting Subspace RGE, yielding an enhanced ZO optimization algorithm, termed \textbf{ZO-Muon}.

\ding{174}  {(Practice)} Compared with the stateful baselines including MeZO \citep{malladi2023finetuning}, LOZO \citep{chen2025enhancing}, SparseMeZO \citep{liu2025sparse}, HiZOO \citep{zhao2025secondorder}, SubZero \citep{yu2025zeroth} and FZOO \citep{dang2025fzoo}, ZO-Muon consistently delivers superior accuracy and efficiency; see Figs.\,\ref{fig:motivating-figure}–\ref{fig:runtime-gpu-compare} for representative results.
\vspace*{-2mm}
\vspace*{-1mm}
\section{Related Work}
\label{sec: related_work}
\vspace*{-3mm}

This work connects three lines of research: (i) ZO optimization and its applications to deep models, (ii) subspace optimization for memory- and compute-efficient training, and (iii) emerging gradient-orthogonalized optimizers for large-scale model training. See detailed review in \textbf{Appendix\,\ref{appendix:related_work}}. 
Among existing studies, the most closely related is JAGUAR-Muon~\citep{petrov2025leveraging}, which attempts to directly apply gradient orthogonalization (GO) to full-space ZO momentum, but incurs additional memory and runtime overhead over MeZO \citep{malladi2023finetuning}, and shows very limited empirical gains.
This is likely because applying GO to noisy, high-dimensional ZO estimates amplifies non-informative spectral components, a limitation also confirmed by our work. By contrast, we show that GO becomes effective for ZO \emph{only when coupled with a subspace formulation}, which we term \emph{subspace gradient orthogonalization}.

\vspace*{-4mm} 
\section{Rethinking ZO Optimization: A Projected Subspace View}
\vspace*{-3mm} 
\label{sec:formulation}

In this section, we first introduce the standard ZO optimization formulation, whose core algorithmic component is a query-based ZO gradient estimator built from finite differences of function values. We then develop a new projection-based subspace perspective, from which we derive a subspace ZO gradient estimator and motivate its effectiveness for improving ZO accuracy and query efficiency.

\noindent \textbf{Standard ZO optimization in the full space.}
Consider a standard minimization problem: 
\begin{align}
\begin{array}{ll}
\displaystyle \minimize_{\mathbf{X} } & f(\mathbf{X}),
\end{array}
\label{eq:prob_0}
\end{align}
where, instead of vectorizing the variable $\mathbf{X} \in \mathbb{R}^{m \times n}$ (\textit{e.g.}, model weights or weight perturbations to a pre-trained model for fine-tuning), we use its matrix form for notational consistency throughout the paper. Here, $f$ denotes the objective function (\textit{e.g.}, the fine-tuning loss for a downstream task).

To solve problem~\eqref{eq:prob_0}, the standard ZO optimization approach replaces the  gradient used in FO methods  with a query-based ZO gradient estimator constructed from finite differences of function values~\citep{liu2020primer}.
A standard choice is the \emph{randomized gradient estimator} (\textbf{RGE}) \citep{nesterov2017random}, whose effectiveness for memory-efficient LLM fine-tuning has been demonstrated by MeZO \citep{malladi2023finetuning}, which is built upon ZO-SGD.
\textit{A generic form of the RGE} \citep{duchi2015optimal} is given by
\begin{align}
    \hat{\nabla}_{\mathbf{X}} f(\mathbf{X}) 
   = \frac{1}{N_{\mathrm{q}}} \sum_{i=1}^{N_{\mathrm{q}}} \frac{f(\mathbf{X} + \mu \boldsymbol{\Psi}_i) - f(\mathbf{X})}{\mu}\,\boldsymbol{\Psi}_i,
    \label{eq:RGEv0}
\end{align}
where $\hat{\nabla}_{\mathbf{X}} f(\mathbf{X})$ denotes the estimated gradient of $f$ with respect to (w.r.t.) $\mathbf{X}$, 
$N_{\mathrm{q}}$ denotes the number of queries, $\boldsymbol{\Psi}_i \in \mathbb{R}^{m \times n}$ is a random perturbation matrix applied to the optimization variable $\mathbf{X}$ whose entries are independently drawn from the standard Gaussian distribution, and $\mu > 0$ is a small smoothing (perturbation) parameter.
As shown in~\eqref{eq:RGEv0}, the RGE is constructed using the \emph{forward-difference} approximation of function values for each query. An alternative variant adopts the \emph{central-difference} scheme,
$
\frac{f(\mathbf{X} + \mu \boldsymbol{\Psi}) - f(\mathbf{X} - \mu \boldsymbol{\Psi})}{2\mu}\,\boldsymbol{\Psi}
$,
which is used by MeZO~\citep{malladi2023finetuning}. In this work, we prefer the forward-difference estimator when $N_{\mathrm{q}} > 1$ due to its lower query complexity:  the RGE in~\eqref{eq:RGEv0} requires only $(N_{\mathrm{q}} + 1)$ function queries, compared to $2N_{\mathrm{q}}$ queries for the central-difference scheme.
When $N_{\mathrm{q}} = 1$, the query complexity of the two schemes is the same; therefore, we use the central-difference estimator.

The RGE \eqref{eq:RGEv0} provides an \textit{unbiased} estimator of the gradient of a \textit{smoothed} version of the original objective, \textit{i.e.}, $\mathbb{E}_{\boldsymbol{\Psi}}\!\left[ f(\mathbf{X} + \mu \boldsymbol{\Psi}) \right]$, and its estimation variance scales as $O\!\left({mn}/{N_{\mathrm{q}}}\right)$ \citep{liu2020primer}.
Consequently, a fundamental limitation of RGE is its high gradient estimation variance, which scales proportionally with the problem dimension and can become prohibitively large for training large-scale models. In theory, mitigating this variance requires using a large number of queries $N_{\mathrm{q}}$, leading to an unfavorable variance-query trade-off.

\noindent \textbf{Subspace ZO optimization: A \textit{projection}-based view.}
A growing body of evidence suggests that, in large model training (especially fine-tuning), the weights and their gradients evolve within a much lower-dimensional subspace \citep{aghajanyan2021intrinsic,hu2022lora,zhao2024galore,kaushik2025universal}. 
Therefore, it is natural to ask \textit{whether one can transition ZO optimization from the full space to a  subspace} since the reduced effective dimension can significantly lower the variance of RGE.

Motivated by the above, we propose to approximate the original problem~\eqref{eq:prob_0} via a projected variant, under the implicit assumption that the optimization variable $\mathbf{X}$ admits a low-dimensional representation   after projection.
To be concrete, let $\mathbf{P} \in \mathbb{R}^{m \times r}$ denote a projection matrix whose $r$ columns (with $r \ll \min \{ m, n \}$)  are \textit{orthonormal}, \textit{i.e.}, $\mathbf{P}^\top \mathbf{P} = \mathbf{I}_r$, where $\mathbf{I}_r \in \mathbb{R}^{r \times r}$ denotes the identity matrix.
We then represent the original variable $\mathbf{X}$ using a lower-dimensional representation $\mathbf{Z} \in \mathbb{R}^{r \times n}$ over the subspace spanned by the columns of $\mathbf{P}$, yielding the approximation $\mathbf X \approx \mathbf P \mathbf Z$. We choose $\mathbf{P}$ to be a \emph{random column-orthonormal} projection matrix (that can be readily obtained via QR decomposition of a random Gaussian matrix) and approximate \eqref{eq:prob_0} as:
\begin{align}
   \displaystyle \minimize_{\mathbf{Z} \in \mathbb{R}^{r\times n}} \; \mathbb{E}_{\mathbf{P}}\!\left[ f(\mathbf{P}\mathbf{Z}) \right].
   \label{eq:subspace_prob}
\end{align}
The above provides a $\mathbf{P}$-projected view of \eqref{eq:prob_0} that can be used to guide effective ZO optimization. 

To further understand the relationship between the   problem~\eqref{eq:prob_0} and its projected approximation~\eqref{eq:subspace_prob}, we examine the FO (stochastic)  gradient of the latter w.r.t. the lower-dimensional variable $\mathbf{Z}$:
\begin{align}
   \underbrace{ \nabla_{\mathbf{Z}} f(\mathbf{P}\mathbf{Z}) = \mathbf{P}^\top \nabla_{\mathbf{X}} f(\mathbf{X}) }_{\text{\ding{172} gradient in $\mathbf{Z}$ space}}, ~~ 
   \underbrace{ \mathbf{P} \nabla_{\mathbf{Z}} f(\mathbf{P}\mathbf{Z}) = \mathbf{P}\mathbf{P}^\top \nabla_{\mathbf{X}} f(\mathbf{X}) }_{\text{\ding{173} gradient lifting back to $\mathbf{X}$ space}}.
   \label{eq:proj_sub_grad_FO}
\end{align}
Here, the transition from \ding{172} to \ding{173} indicates that the subspace gradient $\nabla_{\mathbf{Z}} f(\mathbf{P}\mathbf{Z})$, when lifted back to the full space via the same projection matrix $\mathbf{P}$, yields a projected full-space gradient $\mathbf{P}\mathbf{P}^\top \nabla_{\mathbf{X}} f(\mathbf{X})$. Defining $\hat{\mathbf{P}} \Def \mathbf{P}\mathbf{P}^\top \in \mathbb R^{m \times m}$, we note that $\hat{\mathbf{P}}$ is a projection matrix that maps any vector onto the column space of $\mathbf{P}$ and satisfies the idempotence property $\hat{\mathbf{P}}^2 = \hat{\mathbf{P}}$ (due to $\mathbf{P}^\top \mathbf{P} = \mathbf{I}$). 

The key implications of~\eqref{eq:proj_sub_grad_FO} for ZO optimization are as follows. Under the projected formulation~\eqref{eq:subspace_prob}, we can apply the RGE in~\eqref{eq:RGEv0} directly in the reduced $\mathbf{Z}$-space to obtain a dimension-reduced (and hence variance-reduced) ZO gradient estimate $\hat{\nabla}_{\mathbf{Z}} f(\mathbf{P}\mathbf{Z})$, corresponding to the ZO realization of step~\ding{172} in~\eqref{eq:proj_sub_grad_FO}. This is  given by
\begin{align}
    \hat{\nabla}_{\mathbf{Z}} f(\mathbf{P}\mathbf Z) 
    = \frac{1}{N_{\mathrm{q}}} \sum_{i=1}^{N_{\mathrm{q}}} \frac{f(\mathbf{X} + \mu \mathbf P {\boldsymbol{\Psi}}_i) - f(\mathbf{X})}{\mu}\,\boldsymbol{\Psi}_i ,
    \label{eq:RGE_subspace}
\end{align}
where in contrast to the standard RGE in~\eqref{eq:RGEv0}, the perturbation $\boldsymbol{\Psi}_i \in \mathbb{R}^{r \times n}$ is applied to the lower-dimensional subspace variable $\mathbf{Z}$, and we write the function value $f(\mathbf{P}\mathbf{Z})$ using the exact correspondence  between $\mathbf{P}\mathbf{Z}$ and $\mathbf{X} $. 

The advantage of the ZO gradient estimator in \eqref{eq:RGE_subspace} w.r.t. $\mathbf{Z}$ lies in its reduced estimation variance.
To obtain a full-space gradient estimate, we next leverage step~\ding{173} in~\eqref{eq:proj_sub_grad_FO} to lift the subspace ZO gradient in~\eqref{eq:RGE_subspace} back to the original space, yielding a ZO gradient estimate  w.r.t.  $\mathbf{X}$, which we term \textbf{Subspace RGE} due to its reliance on~\eqref{eq:RGE_subspace}:
\begin{align}
  \displaystyle
\text{\resizebox{.95\textwidth}{!}{$
    \hat{\nabla}_{\mathbf{X}} f(\mathbf X) 
    \approx \mathbf P \hat{\nabla}_{\mathbf{Z}} f(\mathbf{P}\mathbf Z)  
  = \mathbf P   \left [ \frac{1}{N_{\mathrm{q}}} \sum_{i=1}^{N_{\mathrm{q}}} \frac{f(\mathbf{X} + \mu \mathbf P \boldsymbol{\Psi}_i) - f(\mathbf{X})}{\mu}\,\boldsymbol{\Psi}_i \right ]
   =  \frac{1}{N_{\mathrm{q}}} \sum_{i=1}^{N_{\mathrm{q}}} \frac{f(\mathbf{X} + \mu \mathbf P \boldsymbol{\Psi}_i) - f(\mathbf{X})}{\mu}\, \mathbf P \boldsymbol{\Psi}_i
   $.}}
   \label{eq:RGE_subspace_X}
\end{align}

\begin{wrapfigure}{r}{0.65\columnwidth} 
    \vspace*{-5mm}
    \begin{tabular}{ccc}
        \centering   
     \hspace*{-2mm}   \includegraphics[width=0.2\columnwidth]{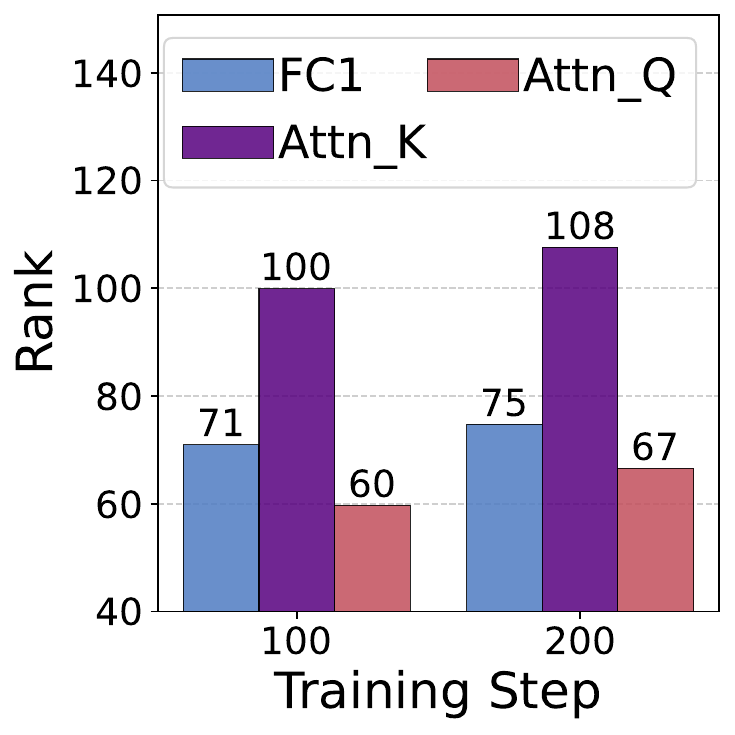} &      \hspace*{-1mm} 
        \includegraphics[width=0.2\columnwidth]{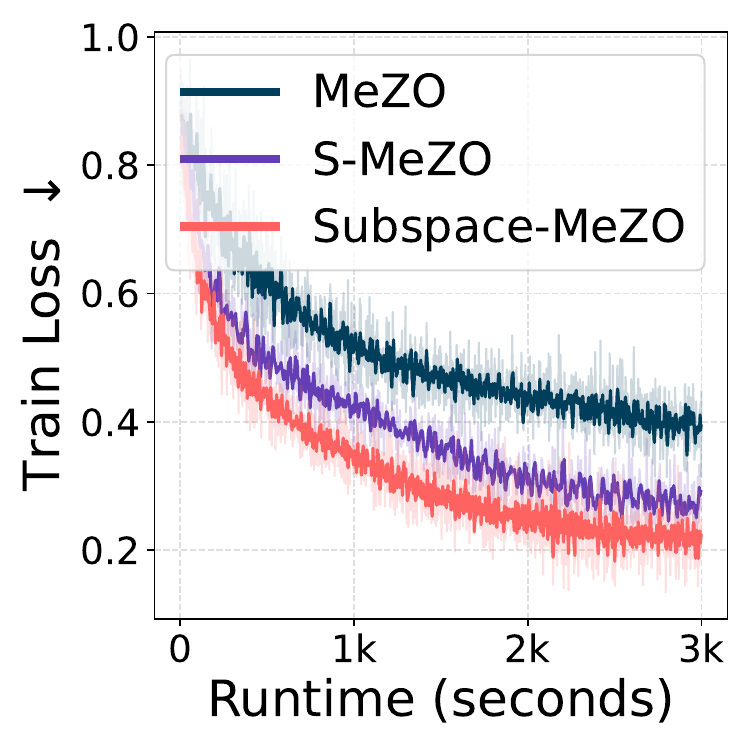} &      \hspace*{-1mm} 
        \includegraphics[width=0.2\columnwidth]{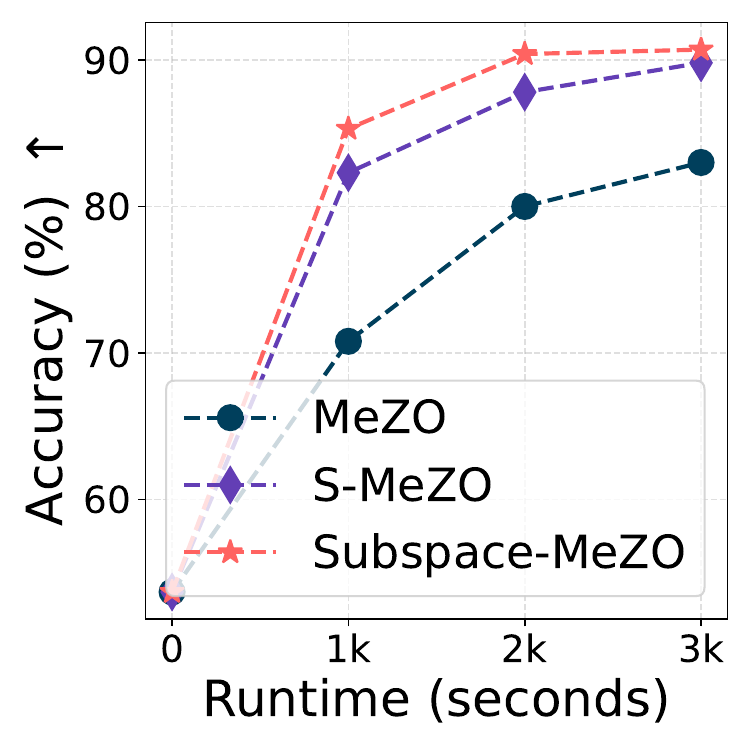} \vspace*{-1.5mm}\\
    \footnotesize{(a) Rank at FO training} &  \hspace*{-1mm} 
        \footnotesize{(b) Training loss} & \hspace*{-1mm} 
        \footnotesize{(c) Fine-tune accuracy} \\
    \end{tabular}
    \vspace*{-2.5mm}
    \caption{\small{The advantages of Subspace RGE, illustrated by fine-tuning OPT-1.3B on SST-2: (a) The low-rank structure of model components during FO training with Adam across steps; (b) Performance of Subspace RGE–based MeZO (\textit{i.e.}, Subspace-MeZO) vs. the baselines SparseMeZO (\textit{i.e.}, S-MeZO) and  MeZO; And (c)    accuracy vs. runtime.  
    }}
    \vspace*{-5mm}
    \label{fig:RGE_subspace}
\end{wrapfigure}

As indicated by \eqref{eq:proj_sub_grad_FO},  the Subspace RGE \eqref{eq:RGE_subspace_X} serves as the ZO counterpart of the projected FO gradient $\mathbf{P}\mathbf{P}^\top \nabla_{\mathbf{X}} f(\mathbf{X})$, rather than of the full gradient.
This seems to introduce an estimation bias at first glance, yet in practice it can actually lead to improved performance due to two key factors: (i) the variance reduction benefit enjoyed by \eqref{eq:RGE_subspace}, and (ii) the intrinsic subspace structure of gradients during model training which makes subspace projection effectively lossless, as will be validated below.

\textbf{Fig.\,\ref{fig:RGE_subspace}} provides a motivating example based on fine-tuning OPT-1.3B on the SST-2 task, which illustrates the above rationale for using the Subspace RGE in~\eqref{eq:RGE_subspace_X} instead of the vanilla RGE in~\eqref{eq:RGEv0}.
Fig.\,\ref{fig:RGE_subspace}(a) shows that the FO gradients during Adam-based training exhibit a pronounced low-rank structure, where the effective rank is measured by the number of top singular values required to capture $99.99\%$ of the spectral energy (see \textbf{Appendix~\ref{appendix:intrinsic-rank}} for details). 
In addition, 
we refer to the integration of Subspace RGE (at $N_q = 1$) into MeZO as \textbf{Subspace-MeZO}, and include MeZO \citep{malladi2023finetuning} and SparseMeZO \citep{liu2025sparse} as baselines; see \textbf{Sec.~\ref{sec: experiment-setup}} for experimental setups.
As shown, Subspace-MeZO achieves both faster convergence (Fig.\,\ref{fig:RGE_subspace}(b)) and better accuracy (Fig.\,\ref{fig:RGE_subspace}(c)).

\noindent \textbf{Comparison with prior work:  A closer look at Subspace RGE.}
The last equality in~\eqref{eq:RGE_subspace_X} suggests that Subspace RGE can be implemented in the same form as the standard RGE \eqref{eq:RGEv0}, but using $\mathbf{P}\boldsymbol{\Psi}$ as the effective perturbation which admits a low-rank factorization into $\mathbf{P} \in \mathbb{R}^{m \times r}$ and $\boldsymbol{\Psi} \in \mathbb{R}^{r \times n}$ since $r \ll \min\{m, n\}$. Here, we omit the query index of $\boldsymbol{\Psi}$ for brevity. 
This perspective connects the proposed Subspace RGE to the existing \underline{lo}w-rank \underline{ZO} optimization approach,  \textbf{LOZO}~\citep{chen2025enhancing}.
We provide a detailed comparison between our proposal and LOZO below.

LOZO~\citep{chen2025enhancing} incorporates low-rank perturbations into the (central-difference) RGE in~\eqref{eq:RGEv0}, yielding the \emph{low-rank gradient estimator} (LGE).
\begin{align}
   & \hat{\nabla}_{\mathbf{X}} f(\mathbf{X}) = \frac{f(\mathbf{X} + \mu  \mathbf{A}\mathbf{B}) - f(\mathbf{X} - \mu  \mathbf{A}\mathbf{B})}{2\mu}  \mathbf{A}\mathbf{B},
    \label{eq:LOZO}
\end{align}
where $\mathbf{A} \in \mathbb{R}^{m \times r}$ and $\mathbf{B} \in \mathbb{R}^{r \times n}$ are two random matrices from standard Gaussian distributions with $r \ll \min\{m, n\}$. The overall perturbation $\mathbf{A}\mathbf{B}$ thus gives a low-rank, factorized approximation of the full-space perturbation $\boldsymbol{\Psi}$ in~\eqref{eq:RGEv0}.

A comparison between the LGE in~\eqref{eq:LOZO} and the proposed Subspace RGE in~\eqref{eq:RGE_subspace_X} shows that the two estimators take essentially the same form, up to a swap between the left perturbation matrices $\mathbf{A}$ and $\mathbf{P}$, and between the right perturbation matrices $\mathbf{B}$ and $\boldsymbol{\Psi}$. However, beyond this formal similarity, there exists a \emph{key difference} in the choice of the (left) perturbation matrix: 
Subspace RGE enforces $\mathbf{P}$ to be \textit{column-orthonormal}, whereas LGE imposes no such constraint on $\mathbf{A}$.
This distinction is critical because $\mathbf{P}$ in Subspace RGE admits a natural interpretation as a subspace projector, which is precisely why it estimates the projected gradient $\hat{\mathbf{P}} \nabla_{\mathbf{X}} f(\mathbf{X})$ in step~\ding{173} of~\eqref{eq:proj_sub_grad_FO}, where $\hat{\mathbf{P}} = \mathbf{P}\mathbf{P}^\top$. 
Recall that $\hat{\mathbf{P}}$ is a projection matrix satisfying the idempotence property, and it is easy to verify that $(\mathbf{I} - \hat{\mathbf{P}})$ is also a projection matrix onto the subspace orthogonal to the column space of $\mathbf{P}$. 
The above indicates that the full gradient naturally admits the decomposition
$
\nabla_{\mathbf{X}} f(\mathbf{X}) 
= \hat{\mathbf{P}} \nabla_{\mathbf{X}} f(\mathbf{X}) 
+ (\mathbf{I} - \hat{\mathbf{P}})\nabla_{\mathbf{X}} f(\mathbf{X})
$, 
where $(\mathbf{I} - \hat{\mathbf{P}})$ is the complementary projection.
Therefore, the use of a column-orthonormal matrix $\mathbf{P}$ in Subspace RGE endows the ZO method with a clear geometric interpretation about the FO gradient component it aims to estimate. By contrast, LOZO does not admit such an interpretation.

\begin{wrapfigure}{r}{0.55\columnwidth}
    \vspace*{-5mm}
    \begin{tabular}{c  c}
        \includegraphics[width=0.25\columnwidth]{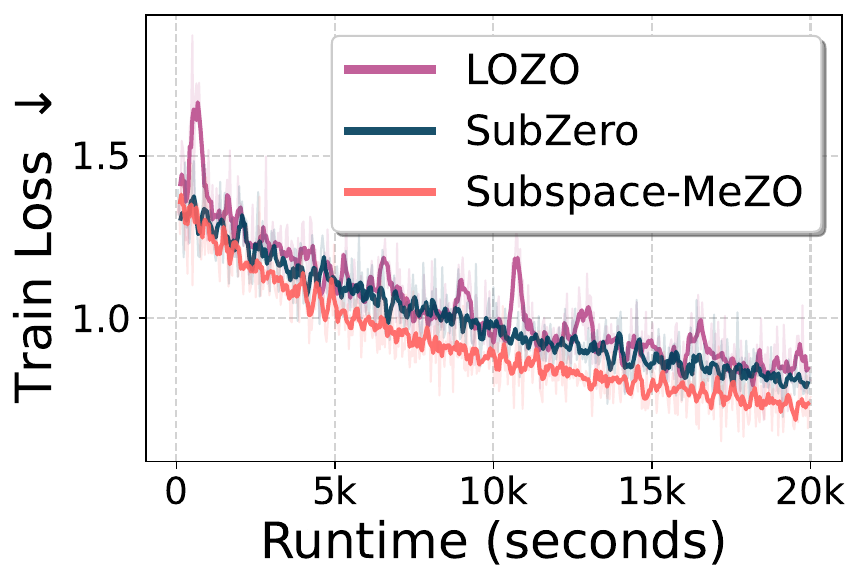} & 
        \includegraphics[width=0.25\columnwidth]{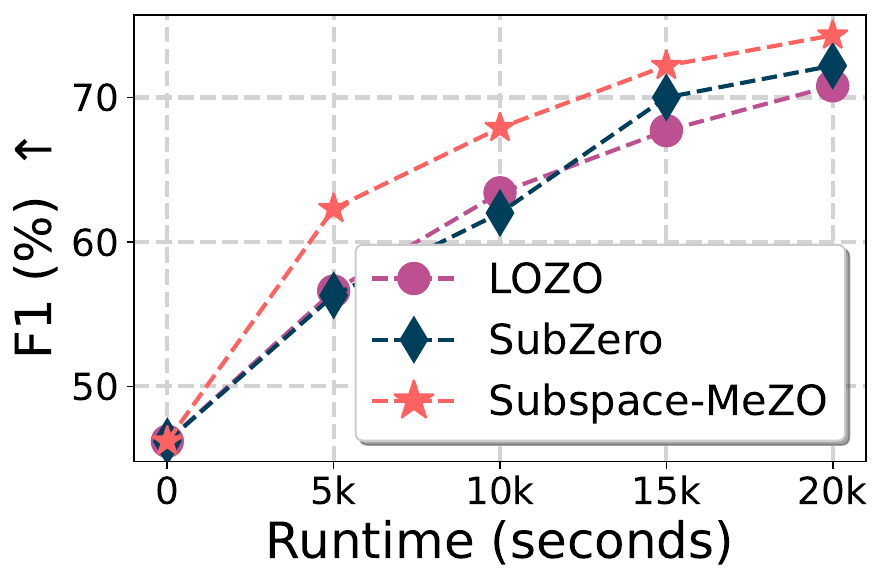} \vspace*{-1.5mm}\\
         \footnotesize{(a) Training loss vs. runtime} & 
        \footnotesize{(b) Fine-tune performance} \\
    \end{tabular}
    \vspace*{-2.5mm}
    \caption{\small{Training (a) and testing (b) performance comparison of the proposed Subspace RGE–based MeZO (\textit{i.e.}, Subspace-MeZO) with existing low-rank perturbation–based ZO baselines, LOZO \citep{chen2025enhancing} and SubZero \citep{yu2025zeroth}, for fine-tuning OPT-13B on SQuAD whose performance is measured by F1.
    }}
    \vspace*{-5mm}
    \label{fig:subspace-mezo}
\end{wrapfigure}
In addition to LOZO, SubZero~\citep{yu2025zeroth} also develops a low-rank perturbation-based RGE similar to~\eqref{eq:LOZO}.
Although, like us, SubZero considers column-orthonormal projections motivated by subspace optimization \citep{nozawa2025zeroth,kozak2023zeroth}, its use of projection is conceptually different from our viewpoint. In our framework, as formalized by steps~\ding{172}--\ding{173} and the derivation of~\eqref{eq:RGE_subspace_X}, we treat subspace optimization as an explicit intermediate step: we first perform RGE in a reduced subspace and then lift the resulting gradient estimate back to the full space to solve the original ZO problem. 
Furthermore, SubZero uses two projection matrices, whose QR-based construction introduces nontrivial runtime overhead and may degrade accuracy due to double projection. 
In \textbf{Fig.\,\ref{fig:subspace-mezo}}, we validate the advantage of the proposed subspace RGE within the Subspace-MeZO framework when fine-tuning OPT-13B on the SQuAD task.
As shown, even when compared against LOZO and SubZero, Subspace-MeZO achieves the fastest convergence (Fig.\,\ref{fig:subspace-mezo}(a)) and the best fine-tune performance (Fig.\,\ref{fig:subspace-mezo}(b)) with the same runtime.



\vspace*{-3mm}
\section{ZO-Muon: Boosting Subspace ZO via Gradient Orthogonalization}
\vspace*{-3mm}
\label{sec:zo-muon-formulation}
In this section, we show that \emph{gradient orthogonalization} (\textbf{GO}), the core ingredient underlying the recently popular optimizer Muon~ \citep{jordan2024muon,liu2025muon}, can be used to further boost the subspace ZO method.


\noindent \textbf{Gradient orthogonalization in Muon: From FO to ZO, with insights from the failure of ZO-Muon-V0.}
Before introducing ZO-Muon, we briefly recall the key idea behind FO Muon. Muon can be interpreted as a steepest-descent method under a spectral-norm constraint \citep{bernstein2024old}, or more broadly as a spectral optimization method that exploits the \emph{matrix-level spectral structure} of descent directions rather than their \emph{entry-wise} information.
Concretely, given the iterate $\mathbf{X}_t$ of the optimization variable $\mathbf{X}$ at iteration $t$ and the learning rate $\eta_t$, the \textbf{basic Muon} (without momentum) takes the form \citep{jordan2024muon}:
\begin{align}
    \mathbf{X}_{t+1} = \mathbf{X}_{t} - \eta_t \, \mathrm{msign}(\mathbf{G}_t),  
    \label{eq:Moun}
\end{align}
where $\mathbf{G}_t $ denotes the (stochastic) gradient at iteration $t$, and $\mathrm{msign}(\cdot)$ denotes the matrix sign function, also known as matrix whitening transformation. The quantity $\mathrm{msign}(\mathbf{G}_t)$ can be interpreted as the descent direction obtained by projecting $\mathbf{G}_t$ onto the spectral-norm unit ball \citep{khaled2025muonbp}, which admits the  closed-form expression:
\begin{align}
   \mathrm{msign}(\mathbf G) 
   = \mathbf G (\mathbf G^{\top} \mathbf G)^{-\frac{1}{2}} 
  = \mathbf U \, \mathrm{sign}(\boldsymbol{\Sigma}) \, \mathbf V^\top
  = \mathbf U_{[:,:k]} \, \mathbf V_{[:,:k]}^\top,
      \label{eq:matrix_sign}
\end{align}
where $\mathbf{G} = \mathbf{U}\boldsymbol{\Sigma}\mathbf{V}^\top$ is the singular value decomposition (SVD) of $\mathbf{G}$, and $\mathrm{sign}(\cdot)$ denotes the (entry-wise) sign function, returning $1$ for non-zero (positive) singular values and $0$ otherwise. Moreover, $\mathbf{U}_{[:, :k]} \in \mathbb{R}^{m \times k}$ and $\mathbf{V}_{[:, :k]} \in \mathbb{R}^{n \times k}$ denote the submatrices of $\mathbf{U}$ and $\mathbf{V}$ consisting of their first $k$ columns, respectively, where $k$ is the \textit{rank} of $\mathbf{G}$.

As shown in \eqref{eq:matrix_sign}, the $\mathrm{msign}$ operation plays the role of GO by equalizing all active singular directions and producing a spectrally isotropic update. This leads to the key {benefit} in Muon, which enables the model update to fully explore and exploit the spectral descent directions. In addition, Muon provides an efficient numerical approximation of $\mathrm{msign}(\cdot)$ via Newton--Schulz (\textbf{NS}) iterations \citep{bernstein2024old}, avoiding the need to compute a full SVD.

Inspired by Muon’s spectral advantage, one may naturally ask whether exploiting such spectral information in the ZO gradient estimator can also bring similar benefits as in FO optimization. Accordingly, by replacing $\mathbf{G}_t$ with its standard RGE counterpart in~\eqref{eq:RGEv0},  we obtain a vanilla ZO-Muon variant, which we denote as \textbf{ZO-Muon-V0}, \textit{i.e.}, applying GO directly to MeZO.


\begin{wrapfigure}{r}{0.25\columnwidth} 
    \centering
    \vspace*{-5mm}   
\includegraphics[width=\linewidth]{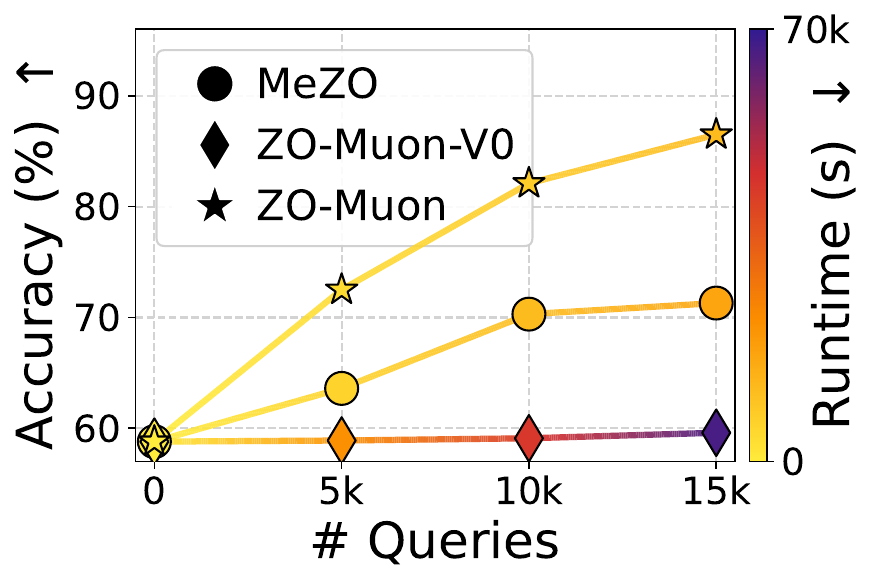} 
    \vspace*{-6mm}
    \caption{\small{Demonstration of the ineffectiveness of ZO-Muon-V0 vs. MeZO and the effectiveness of ZO-Muon through fine-tuning OPT-13B on SST-2.}}
    \label{fig:ZO_Muon_V0}
    \vspace*{-5mm}
\end{wrapfigure}
However, we find that ZO-Muon-V0 does \textit{not} provide additional benefits over the GO-\textit{absent} baseline, MeZO. As shown in \textbf{Fig.\,\ref{fig:ZO_Muon_V0}}, when fine-tuning OPT models on SST-2, ZO-Muon-V0 \textit{underperforms} MeZO in both accuracy and runtime efficiency under the same query budget.
This \textit{accuracy} underperformance is likely because directly applying GO to noisy ZO gradient estimates amplifies spurious directions during matrix whitening, thereby degrading optimization performance.
From a \textit{runtime} perspective, performing the NS iterations for GO also introduces additional computational overhead compared to MeZO, \textit{i.e.}, RGE-based ZO-SGD. In addition to ZO-Muon-V0, we also evaluate the ineffectiveness of {JAGUAR-Muon} \citep{petrov2025leveraging}, another direct attempt to integrate Muon into ZO optimization; see \textbf{Fig.\,\ref{fig:full-space-appendix}} in \textbf{Appendix\,\ref{appendix:ineffectiveness}}.

\noindent \textbf{ZO-Muon: How and why.}
Effectively harnessing Muon’s spectral advantages in the ZO setting is \emph{nontrivial}. We show that the subspace ZO formulation (Sec.\,\ref{sec:formulation}) provides the essential structure to realize these benefits; This is \textit{not a simple integration}, but a necessary design choice. Specifically,   subspace RGE \eqref{eq:RGE_subspace} produces higher-quality, variance-reduced gradient estimates before GO. This is critical, as GO induces spectrally isotropic updates that can otherwise amplify noise in low-quality ZO estimates, leading to degraded performance (Fig.\,\ref{fig:ZO_Muon_V0}). Moreover, since the cost of the $\mathrm{msign}$ operation in \eqref{eq:matrix_sign} scales with matrix rank \citep{he2025low}, the (low-rank) subspace structure also improves efficiency. As validated by ZO-Muon in Fig.\,\ref{fig:ZO_Muon_V0}, this synergy yields a clear accuracy–efficiency win–win.


We next formally develop ZO-Muon.
Recalling the projection-based view used to derive Subspace RGE in~\eqref{eq:RGE_subspace_X} from~\eqref{eq:proj_sub_grad_FO}, we can integrate GO (\textit{i.e.}, the $\mathrm{msign}$ operation) with the subspace RGE in~\eqref{eq:RGE_subspace} on the lower-dimensional variable $\mathbf{Z}$, and then lift the resulting gradient-orthogonalized RGE back to the full space w.r.t. the original variable $\mathbf{X}$.
This leads to \textbf{ZO-Muon} that modifies the Muon \eqref{eq:Moun} to
\begin{align}
      \hspace*{-3mm} \mathbf O_t = \overbrace{ \mathbf P  \mathrm{msign}\left ( \underbrace { \frac{1}{N_{\mathrm{q}}} \sum_{i=1}^{N_{\mathrm{q}}} \frac{f(\mathbf{X}_t + \mu \mathbf P \boldsymbol{\Psi}_i) - f(\mathbf{X}_t)}{\mu}\,\boldsymbol{\Psi}_i }_\text{RGE $ \hat{\nabla}_{\mathbf{Z}} f$ within subspace in \eqref{eq:RGE_subspace}, like step \ding{172} in \eqref{eq:proj_sub_grad_FO} } \right )}^\text{Lifting GO update back to $\mathbf{X}$ space, like step \ding{173} in \eqref{eq:proj_sub_grad_FO} } , \quad 
      \mathbf{X}_{t+1} = \mathbf{X}_{t} - \eta_t \, \mathbf O_t.
      \label{eq:ZO_Muon}
\end{align}
Here, 
compared to \eqref{eq:Moun}, 
$\mathbf{O}_t$ can be understood as the projection-based realization of the gradient-orthogonalized subspace RGE.
See a full algorithmic description of   ZO-Muon   in \textbf{Appendix\,\ref{appendix:zo-muon-formulation}}.



Furthermore, we provide a theoretical underpinning of the proposed ZO-Muon \eqref{eq:ZO_Muon} from a  low-rank  Muon perspective. As shown, the RGE inside the $\mathrm{msign}$ operator prior to GO corresponds to an approximation of the subspace-projected gradient $\mathbf{P}^\top \nabla_{\mathbf{X}} f(\mathbf{X})$, as given by \eqref{eq:RGE_subspace} and the step~\ding{172} in~\eqref{eq:proj_sub_grad_FO}.
Therefore, the FO interpretation of the ZO-Muon step yields
\begin{align}
    \mathbf O =  \mathbf P  \mathrm{msign} (\mathbf{P}^\top  \mathbf G), ~~ \mathbf G = \nabla_{\mathbf{X}} f(\mathbf{X}),
    \label{eq:FO_GO_proj}
\end{align}
where, for notational simplicity, we omit the iteration index $t$.
Clearly, in the absence of the projection $\mathbf{P} \in \mathbb{R}^{m \times r}$, the step~\eqref{eq:FO_GO_proj} reduces to the standard GO update in Muon \eqref{eq:Moun}.
Since $r \ll \min\{m, n\}$, the GO step  \eqref{eq:FO_GO_proj} is applied to the low-rank matrix $\mathbf{P}^\top \mathbf G$, rather than to the full gradient $\mathbf G \in \mathbb R^{m \times n }$.
Proposition\,\ref{prop:low-rank-GO} below shows \textit{when} the GO step \eqref{eq:FO_GO_proj} becomes equivalent to to the original GO step  \eqref{eq:matrix_sign}.
\vspace*{-2mm}
\begin{myprop}
\label{prop:low-rank-GO}
   If the projection is chosen as $\mathbf{P} = \mathbf{U}_{[:, :k]} \in \mathbb{R}^{m \times k}$, obtained from the SVD of $\mathbf{G}$ in~\eqref{eq:matrix_sign}, then the projection for gradient orthogonalization is \emph{lossless}. That is,
    \begin{align}
       \mathbf P  \mathrm{msign} (\mathbf{P}^\top  \mathbf G) =  \mathrm{msign}(\mathbf G).
       \label{eq:GO_lowrank_proof}
    \end{align}
\end{myprop}
\vspace*{-2mm}
\textbf{Proof}: See details in \textbf{Appendix~\ref{sec:GO_lowrank_appendix}}. 
\hfill $\square$


 Prop.\,\ref{prop:low-rank-GO} implies that ZO-Muon can be viewed as a ZO approximation of low-rank Muon~\eqref{eq:FO_GO_proj}  \citep{he2025low}.  Under standard assumptions for non-convex optimization, \textbf{Appendix\,\ref{appendix:convergence}} shows that, by treating the ZO gradient estimate as an unbiased estimator of the smoothed objective \cite{liu2020primer} and interpreting ZO-Muon as its low-rank solver, convergence to stationarity is guaranteed with a sublinear rate of $\mathcal{O} \left(\frac{1}{\delta}\sqrt{\frac{r}{T}}\right)$,  where $\delta<1$ is the subspace alignment ratio measuring the proportion of the projected gradient relative to the full gradient as defined in \eqref{eq:proj_sub_grad_FO}, $T$ the number of iterations, and $r$ the projection dimension.
This highlights a key trade-off: projection incurs a factor $1/\delta$ due to partial gradient alignment, but significantly reduces estimation variance, yielding improved dependence on the intrinsic dimension $r$ rather than the original variable dimension.
 


\noindent \textbf{Projection matrix  sampling strategy: Why   random is used.}
Based on Prop.\,\ref{prop:low-rank-GO}, the projection matrix ideally comes from the SVD of FO gradients. In ZO settings, however, this is infeasible. Motivated by the empirical success of using a random column-orthonormal projection matrix $\mathbf P$ in Subspace RGE \eqref{eq:RGE_subspace_X}, we adopt a practical strategy of lazily sampling $\mathbf P$ as a random column-orthonormal projection.
While one might think that an informed projection matrix $\mathbf P$ in \eqref{eq:ZO_Muon}, such as one derived from the moving average of past ZO gradients, would better align with the true gradient's principal components, our empirical results suggest otherwise. As detailed in \textbf{Table \ref{tab:projection-matrix-sampling} of Appendix \ref{appendix:additional-results}}, Gaussian sketching~\citep{halko2011finding} to past descending directions yields no notable performance gains over random sampling with either Subspace-MeZO or ZO-Muon. Furthermore, maintaining the historical momentum matrix increases GPU memory consumption, which counteracts the primary memory-efficiency motivation of ZO fine-tuning. 

%

\noindent \textbf{Hyperparameters in ZO-Muon: Query number $N_q$ and projection rank $r$.}
Two hyperparameters play a central role in ZO-Muon: the number of queries $N_q$ in~\eqref{eq:ZO_Muon} and the projection rank $r$ of $\mathbf{P}$.


First, we find that setting the query number $N_q > 1$ is crucial for ZO-Muon, \textit{e.g.}, $N_q = 4$ works well in our experiments, indicating that the RGE in \eqref{eq:ZO_Muon} should be implemented in its multi-query form.
Recall that in the absence of $\mathrm{msign}$, the method reduces to Subspace RGE, yielding Subspace-MeZO as examined in Fig.\,\ref{fig:RGE_subspace}. For both Subspace-MeZO and its full-space baseline MeZO~\citep{malladi2023finetuning}, the default choice is $N_q = 1$.  This is due to  the \emph{multi-query paradox}~\citep{lin2025multi}: Although using multiple queries per iteration can reduce the variance of the gradient estimate, whether it improves optimization performance depends on how the queries are aggregated. In particular, simple averaging with $N_q > 1$ can lead to diminishing returns (\textit{i.e.}, the multi-query paradox) compared to using a single query.

\begin{wrapfigure}{r}{0.6\columnwidth}
\vspace*{-6mm}
    \begin{tabular}{c @{\hspace{2mm}} c}
        \includegraphics[width=0.25\columnwidth]{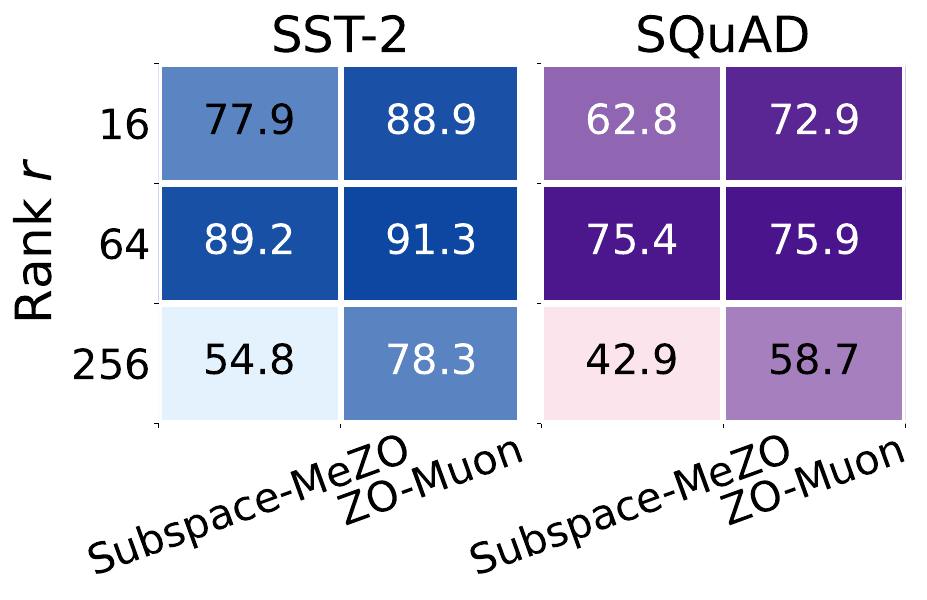} & 
        \includegraphics[width=0.25\columnwidth]{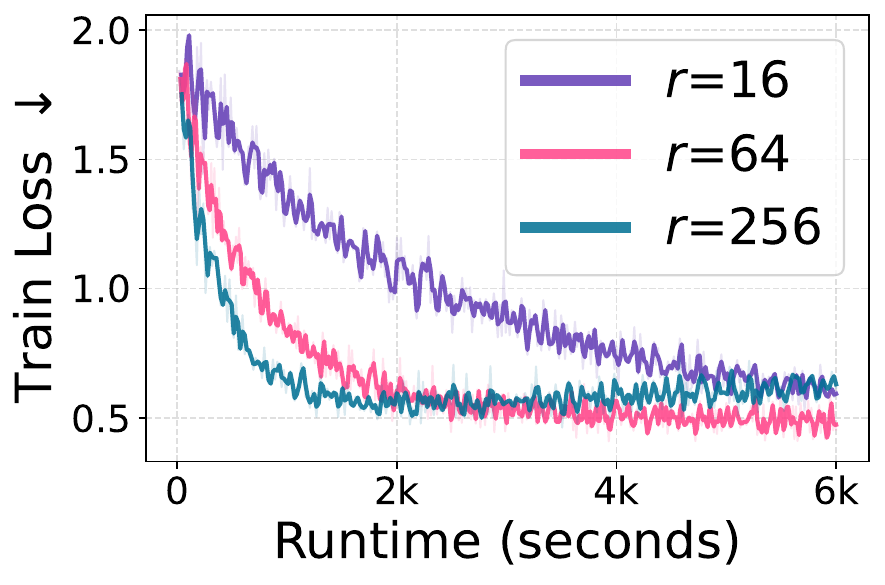} \\
    \end{tabular}
    \vspace*{-4mm}
    \caption{\small{Comparison of ZO-Muon and Subspace-MeZO for fine-tuning OPT-1.3B on SST-2 and SQuAD with $r \in \{16, 64, 256\}$. \textit{(Left)} Accuracy (shown by cell color and numeric value) for different $r$. ZO-Muon consistently outperforms Subspace-MeZO, with $r = 64$ achieving the best performance. \textit{(Right)} Training loss of ZO-Muon under different $r$ for fine-tuning on SQuAD.
    }}
    \vspace*{-5mm}
    \label{fig:rank}
\end{wrapfigure}

However, unlike Subspace-MeZO, when $N_q = 1$ is used in ZO-Muon, the GO step in \eqref{eq:ZO_Muon} reduces to
$\mathrm{msign}(s\boldsymbol{\Psi}) = \mathrm{sign}(s) \mathrm{msign}(\boldsymbol{\Psi})$,
where $s = \frac{f(\mathbf{X} + \mu \mathbf{P}\boldsymbol{\Psi}) - f(\mathbf{X})}{\mu}$ is the finite-difference scalar and $s\boldsymbol{\Psi}$ is the RGE. Here, $\mathrm{sign}(s)$ returns $1$ or $-1$ depending on the sign of $s$.
In this case, even when the RGE magnitude $s$ becomes small (towards convergence), the $\mathrm{msign}$ operation  discards its scale and may amplify small noise  by enforcing a spectrally isotropic update. 
Therefore, ZO-Muon benefits from using $N_q > 1$. 
Meanwhile, it effectively overcomes the multi-query paradox compared to Subspace-MeZO (see \textbf{Appendix\,\ref{appendix:N_q-validation}} for empirical validation).

For the other key hyperparameter, the projection rank $r$ that defines the subspace via $\mathbf{P}$, we find that ZO-Muon remains effective as long as it is set to a reasonably reduced dimension (\textit{e.g.}, $r \in \{ 64, 128\}$). This is roughly aligned with the effective rank of FO gradients shown in Fig.\,\ref{fig:RGE_subspace}(a). To validate this, \textbf{Fig.\,\ref{fig:rank}} compares the performance of ZO-Muon with Subspace-MeZO under different values of $r$ when fine-tuning OPT-1.3B on SST-2 and SQuAD. As shown in \textbf{Fig.\,\ref{fig:rank}(a)}, ZO-Muon consistently outperforms Subspace-MeZO across all rank choices (indicating the usefulness of GO in ZO-Muon) and achieves the best performance with a moderate rank such as $r=64$. Setting $r$ too small or too large degrades the performance of ZO-Muon: A too-small $r$ discards informative gradient components during projection, while an overly large $r$ amplifies noisy directions after GO. This is further supported by \textbf{Fig.\,\ref{fig:rank}(b)}, which shows that ZO-Muon converges slowly when $r$ is too small, whereas using an overly large rank (\textit{e.g.}, $r=256$) leads to instability and loss increase in the later stage of training, likely due to excessive gradient noise in ZO optimization.


\vspace*{-4mm}
\section{Experiments}
\vspace*{-3mm}
\label{sec: experiment}
\subsection{Experiment setups}
\vspace*{-3mm}
\label{sec: experiment-setup}

\noindent \textbf{Models and datasets.} 
We evaluate the performance of ZO on fine-tuning tasks for both LLMs and ViTs. For LLMs, we primarily consider 
LlaMA3-8B \citep{grattafiori2024LlaMA} and OPT-1.3B/13B \citep{zhang2022opt} on the SuperGLUE benchmark \citep{wang2019superglue}. For ViTs, we fine-tune ViT-B/16 and ViT-L/16 \citep{dosovitskiy2021an} on CIFAR-10/100 \citep{krizhevsky2009learning}.

\begin{wraptable}{r}{0.6\textwidth} 
\centering
\small
\vspace*{-4mm}
\caption{\small{ZO-Muon outperforms the ZO baselines by a non-negligible margin, demonstrated by fine-tuning LLMs (LlaMA3-8B and OPT-13B) on SuperGLUE benchmark and fine-tuning ViTs (ViT-B and ViT-L) on CIFAR datasets. 
FO methods (marked by $\orangedot$) include Adam and LoRA. 
The best ZO fine-tuning result for each task is highlighted in \textbf{bold}. 
}}
\label{tab:main-table}
\resizebox{\linewidth}{!}{%
    \begin{threeparttable}
        \begin{tabular}{lcccccc|cc}
        \toprule
        \toprule
        \multirow{2}{*}{\textbf{Method}} & \multicolumn{6}{c}{\textbf{LLM Fine-tuning}} & \multicolumn{2}{c}{\textbf{ViT Fine-tuning}} \\
        \cmidrule(lr){2-7} \cmidrule(lr){8-9}
         & \textbf{SST-2} & \textbf{RTE} & \textbf{CB} & \textbf{BoolQ} & \textbf{WiC} & \textbf{SQuAD} & \textbf{CF10} & \textbf{CF100} \\
        \midrule
        
         & \multicolumn{6}{c}{\cellcolor{lightblue}\textbf{LlaMA3-8B}} & \multicolumn{2}{c}{\cellcolor{lightgreen}\textbf{ViT-B}} \\
        \midrule
        $\orangedot$ Full    & 96.0 & 92.0 & 92.0 & 86.6 & 72.6 & 90.4 & 98.9 & 91.2 \\
        $\orangedot$ LoRA    & 95.0 & 80.9 & 73.2 & 86.4 & 70.7 & 89.4 & 98.3 & 86.4 \\
        $\orangedot$ Galore   & 95.4 & 85.2 & 87.5 & 85.1 & 72.7 & 89.5 & 97.7 & 87.3  \\
        MeZO    & 92.7 & 74.4 & \textbf{69.6} & 76.7 & 57.8 & 86.7 & 93.5 & 64.5 \\
        S-MeZO  & 92.1 & 69.7 & 69.6 & 80.5 & 56.9 & 87.5 & 93.7 & 64.3 \\
        HiZOO   & 93.5 & 75.1 & 69.6 & 80.0 & 59.7 & 87.3 & 93.2 & 64.9 \\
        LOZO    & 92.5 & 66.8 & \textbf{69.6} & 79.4 & 55.8 & \textbf{89.0} & 93.2 & 61.8 \\
        SubZero & 92.1 & 71.4 & 67.9 & 82.0 & 58.8 & 88.3 & 41.1\tnote{*} & 7.3\tnote{*} \\
        Subspace-MeZO & 92.3 & 68.6 & \textbf{69.6} & 80.0 & 62.9 & 84.5 & 94.2 & 58.2 \\
        FZOO & \textbf{94.3} &77.6 & \textbf{69.6} & 81.8 & 60.8 & 87.9 & 8.5\tnote{*} & 1.7\tnote{*} \\
        \textbf{ZO-Muon} & \textbf{94.3} & \textbf{81.2} & \textbf{69.6} & \textbf{82.9} & \textbf{65.2} & 88.2 & \textbf{94.9} & \textbf{72.6} \\
        \midrule
        
         & \multicolumn{6}{c}{\cellcolor{lightblue}\textbf{OPT-13B}} & \multicolumn{2}{c}{\cellcolor{lightgreen}\textbf{ViT-L}} \\
        \midrule
        $\orangedot$ Full  & 95.3 & 80.9 & 94.6 & 83.5 & 66.3 & 89.5 & 98.8 & 91.7 \\
        $\orangedot$ LoRA    & 94.8 & 78.3 & 69.6 & 80.2 & 64.3 & 88.0 & 98.8 & 90.6 \\
        $\orangedot$ Galore    & 95.4 & 78.7 & 71.4 & 79.5 & 63.9  & 88.0 & 97.8 & 90.3 \\
        MeZO    & 91.4 & 66.1 & 66.0 & 66.1 & 59.4 & 81.8 & 95.6 & 47.3 \\
        S-MeZO  & 90.4 & 63.5 & 69.6 & 66.4 & 58.8 & 80.8 & 96.9 & 56.3 \\
        HiZOO   & 92.1 & 69.3 & 69.6 & 67.6 & 59.4 & 82.1 & 96.3 & 46.6 \\
        LOZO    & 91.7 & 70.4 & 69.6 & 71.9 & 60.2 & \textbf{84.9} & 95.3 & 48.9 \\
        SubZero & 92.1 & 71.8 & \textbf{71.4} & 70.8 & 60.8 & 84.5 & 41.2\tnote{*} & 3.2\tnote{*} \\
        Subspace-MeZO & 91.7 & 70.7 & \textbf{71.4} & 68.1 & \textbf{61.7} & 83.5 & 97.0 & 64.7 \\
        FZOO & \textbf{93.7} & 71.7 & 69.6 & 72.2 & 60.5 & 84.8 & 13.0\tnote{*} & 1.1\tnote{*} \\
        \textbf{ZO-Muon} & 92.5 & \textbf{73.3} & \textbf{71.4} & \textbf{72.4} & \textbf{61.7} & 84.5 & \textbf{97.2} & \textbf{72.4} \\
        \bottomrule
        \bottomrule
        \end{tabular}%
        
        \begin{tablenotes}
            \item [*] SubZero and FZOO, originally used for LLM tuning, underperform on vision tasks even after extensive hyperparameter tuning.
        \end{tablenotes}
    \end{threeparttable}
}
\vspace*{-4mm}
\end{wraptable}

\noindent \textbf{Baselines and evaluations.}
We compare the proposed ZO-Muon with the following stateful ZO optimization baselines, which were primarily developed for large-model fine-tuning. These include MeZO \citep{malladi2023finetuning}, SparseMeZO (S-MeZO) \citep{liu2025sparse}, HiZOO \citep{zhao2025secondorder}, LOZO \citep{chen2025enhancing}, SubZero \citep{yu2025zeroth}, FZOO \citep{dang2025fzoo}, and Subspace-MeZO developed in Sec.\,\ref{sec:formulation} prior to integration with Muon.
We also include FO baselines, full parameter fine-tuning (Full), LoRA \citep{hu2022lora}, and Galore \citep{zhao2024galore}. Following previous works \citep{malladi2023finetuning, dosovitskiy2021an}, we use Adam \citep{kingma2014adam} for LLMs and SGD (with momentum) \citep{rumelhart1986learning} for ViTs.

\noindent \textbf{Implementation details.} 
For ZO-Muon, we perform a hyperparameter grid search over the projection rank $r \in \{64, 128\}$ across all settings, following the rank selection rationale in Sec.\,\ref{sec:zo-muon-formulation}. We also search over $N_q \in \{4, 8, 16\}$, consistent with our justification for using $N_q > 1$ in Sec.\,\ref{sec:zo-muon-formulation}. Empirically, we find that setting $N_q = 16$ and $r = 64$ for ViTs, and using $N_q = 4$ with $r \in \{64, 128\}$ for LLMs depending on the task, consistently yields strong performance.
To ensure a fair comparison with other baselines, we scale the number of training steps inversely with $N_q$, so that ZO-Muon does not incur additional query or runtime overhead. The projection matrix $\mathbf{P}$ is resampled via QR decomposition every 100 iterations (\textit{i.e.}, lazily updated). For the $\mathrm{msign}$ operation, we use the NS (Newton--Schulz) iteration unless otherwise specified. 
We set the training batch size to 16 for all LLM experiments, and to 64 and 256 for ViT fine-tuning on CIFAR-10 and CIFAR-100, respectively. The rationale for these choices is supported by the experimental results presented later.
 Additional experiment setups, including settings of baselines, are detailed in Appendix~\ref{appendix:experiment-details}.

\subsection{Experiment results} 
\label{sec:experiment-result}

\noindent \textbf{ZO-Muon achieves superior ZO fine-tuning accuracy for both LLMs and ViTs.}
In \textbf{Tab.\,\ref{tab:main-table}}, we compare the LLM and ViT fine-tuning performance of ZO-Muon with FO and ZO baselines on LlaMA3-8B and OPT-13B, as well as ViT-B and ViT-L, where each column corresponds to a fine-tuning task.
As shown, ZO-Muon consistently improves fine-tuning accuracy over existing ZO baselines. In particular, for higher-complexity tasks such as RTE and WiC in LLM fine-tuning and CIFAR-100 in ViT fine-tuning, ZO-Muon achieves substantial gains, narrowing the gap to FO baselines. Notably, on ViT-L CIFAR-100 fine-tuning, ZO-Muon outperforms the second-best ZO method (Subspace-MeZO) by $7.7\%$ and MeZO by $25.1\%$ in accuracy.
In addition, we observe that Subspace-MeZO consistently ranks among the top-performing ZO baselines, and that ZO-Muon further improves upon it. This indicates that subspace ZO and Muon’s gradient orthogonalization provide complementary benefits for improving overall ZO optimization.

\begin{wraptable}{r}{0.65\textwidth}
\centering
\vspace*{-2mm}
\small
\caption{\small{Efficiency comparison in terms of the number of train steps, query number, runtime, memory cost, GPU usage, and accuracy for fine-tuning LlaMA3-8B on SST-2. Note that the total query budget is fixed to 40k for all ZO methods.}
}
\vspace*{-1mm}
\label{tab:computation-cost-table}
\resizebox{\linewidth}{!}{%
\begin{tabular}{ccccccc}
\toprule
\toprule
\textbf{Method} & \textbf{Steps} $\downarrow$ & \textbf{Queries} $\downarrow$ & \textbf{Runtime} $\downarrow$ & \textbf{Memory} $\downarrow$ & \textbf{GPU}s $\downarrow$ & \textbf{Accuracy} $\uparrow$ \\
\midrule
$\orangedot$ Full & 625 & 625 & 22\text{min} & $4 \times 32.8$ \text{GB} & $4\times \text{A6000}$ & 96.0 \\
$\orangedot$ LoRA & 625 & 625 & 14\text{min} & $2 \times 27.9 $ \text{GB} & $2\times \text{A6000}$ & 95.0 \\
$\orangedot$ Galore & 625 & 625 & 46\text{min} & $4 \times 21.8 $ \text{GB} & $4\times \text{A6000}$ & 95.4 \\
\midrule
  MeZO    & 20k & 40k & 4\text{h} 52\text{min} &  21.0\text{GB} & $1\times \text{A6000}$ & 92.7 \\
S-MeZO    & 20k & 40k & 9\text{h} 8\text{min} &  21.3\text{GB} & $1\times \text{A6000}$ & 92.1 \\
 HiZOO    & 13.3k & 40k & 7\text{h} 10\text{min} &  39.3\text{GB} & $1\times \text{A6000}$ & 93.5 \\
  LOZO    & 20k & 40k & 5\text{h} 4\text{min} &  20.0\text{GB} & $1\times \text{A6000}$ & 92.5 \\
 SubZero    & 20k & 40k & 5\text{h} 34\text{min} &  20.2\text{GB} & $1\times \text{A6000}$ & 92.1 \\
 Subspace-MeZO    & 20k & 40k & 4\text{h} 20\text{min} &   19.2\text{GB} & $1\times \text{A6000}$ & 92.3 \\
  FZOO    & 8k & 40k & 5\text{h} 20\text{min} &   20.1\text{GB} & $1\times \text{A6000}$ & 94.3 \\
 \textbf{ZO-Muon} & 8k & 40k & 3\text{h} 57\text{min} & 19.4 \text{GB} & $1\times \text{A6000}$ & 94.3 \\
\bottomrule
\bottomrule
\end{tabular}%
}
\end{wraptable}

\noindent \textbf{ZO-Muon delivers a strong efficiency profile over other ZO baselines.}
In \textbf{Tab.\,\ref{tab:computation-cost-table}}, we compare the efficiency of ZO-Muon with ZO and FO baselines in terms of the number of optimization steps, number of queries, runtime, memory cost, and GPU usage for fine-tuning LlaMA3-8B on SST-2, with the corresponding accuracy results reported in Tab.~\ref{tab:main-table}.
For a fair comparison, we fix the total query budget for gradient estimation during fine-tuning to 40k for all ZO methods.
As we can see, ZO-Muon significantly reduces the number of ZO optimization steps and thereby yields the improved runtime efficiency compared to other ZO baselines. For example,  ZO-Muon requires only $81.1\%$ of MeZO's runtime while outperforming it by $1.6\%$ in fine-tuning accuracy.
In addition, compared to FO approaches, ZO optimization does not offer a runtime advantage, but it significantly reduces memory cost and GPU requirements.
This is the primary motivation for using ZO fine-tuning as a memory-efficient alternative to FO. 
We observe that ZO-Muon inherits this advantage and does not increase memory cost or GPU usage compared to other ZO baselines. It is worth noting that while Galore requires less memory than LoRA in theory \citep{zhao2024galore}, it consumes higher GPU peak memory in large model fine-tuning, likely due to the SVD operations inovlved in Galore.

\begin{wrapfigure}{r}{0.65\textwidth}
    \vspace*{-5mm}
    \begin{tabular}{cc}
        \includegraphics[width=0.3\textwidth]{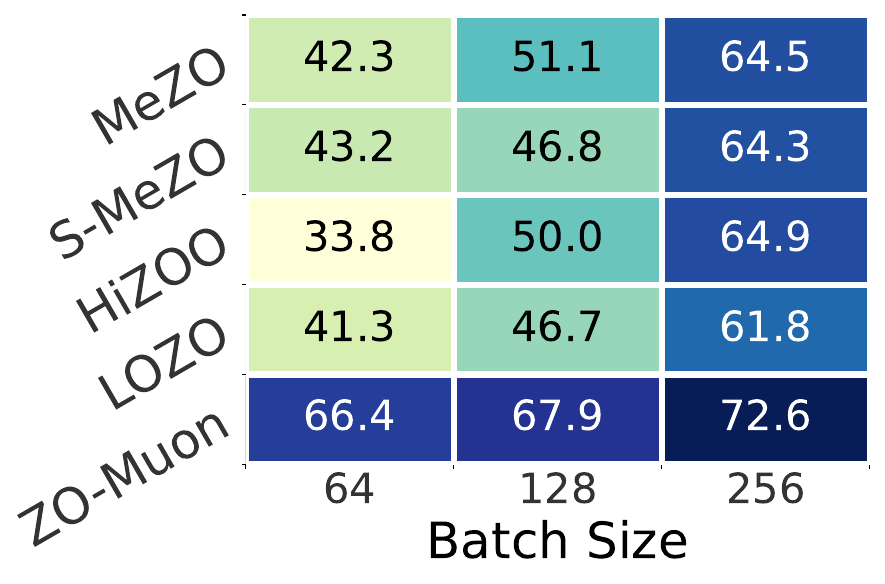} & 
        \includegraphics[width=0.3\textwidth]{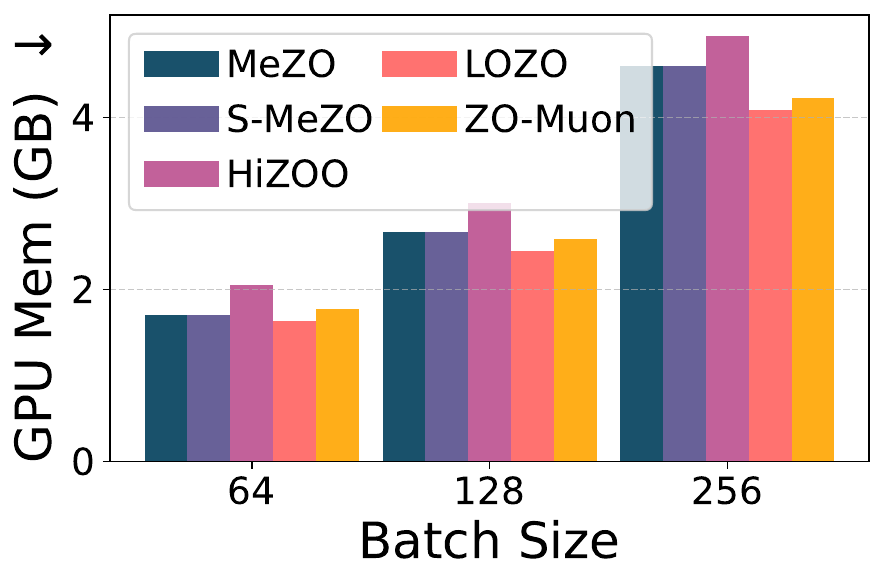}  \vspace*{-2mm}\\
       
        \footnotesize{(a) Fine-tune accuracy} & 
        \footnotesize{(b) GPU memory} \\   
    \end{tabular}  
    \vspace*{-3mm}
    \caption{\small{Accuracy and GPU memory comparison under varying batch sizes. ZO-Muon vs. baselines on ViT-B fine-tuning for CIFAR-100 with a fixed query budget.}}
    \label{fig:vit-bs}
    \vspace*{-3mm}
\end{wrapfigure}

\noindent \textbf{ZO-Muon does \textit{not} rely on large batch sizes.}
In ZO optimization, increasing the batch size reduces gradient-estimation variance but also raises memory usage, since GPU cost is dominated by forward passes. We therefore study ZO performance as a function of batch size. 
\textbf{Fig.\,\ref{fig:vit-bs}} 
compares accuracy and GPU memory of ZO-Muon and baselines for ViT-B fine-tuning on CIFAR-100 under batch sizes 64, 128, and 256. While larger batches generally improve ZO accuracy, reducing the batch size from 256 to 64 causes a severe drop for baselines (\textit{e.g.}, $22.2\%$ for MeZO) but only a modest drop for ZO-Muon ($6.2\%$). Notably, ZO-Muon with batch size 64 still outperforms MeZO with batch size 256. Together with Fig.\,\ref{fig:vit-bs}(b), which shows the higher memory cost of larger batches, ZO-Muon is substantially more memory-efficient.
The same advantage is also observed for LLM fine-tuning  in \textbf{Fig.\,\ref{fig:llm-bs}}.

\begin{figure}[htb]
    \centering
    \vspace*{-3mm}
     \centering
    \hspace{3mm}\includegraphics[width=0.6\columnwidth] {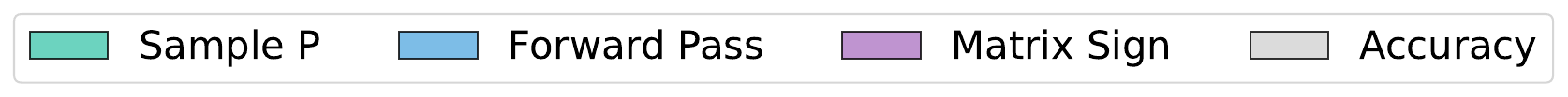}   \\  
    \begin{tabular}{cc}
        \includegraphics[width=0.47\columnwidth]{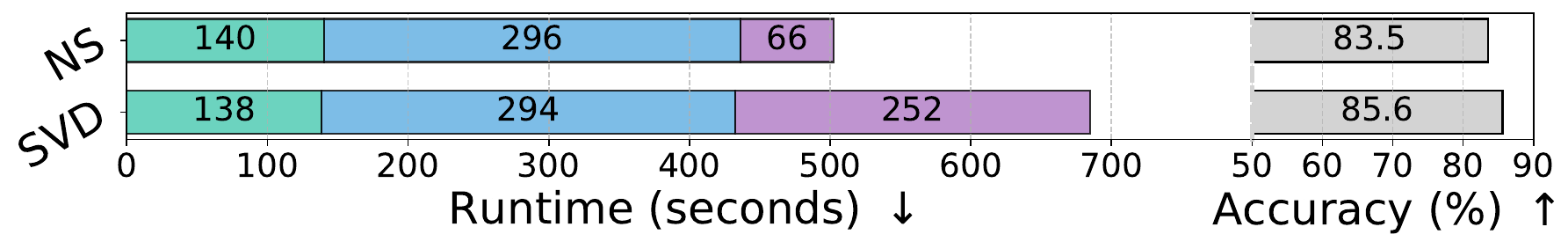} & 
        \includegraphics[width=0.47\columnwidth]{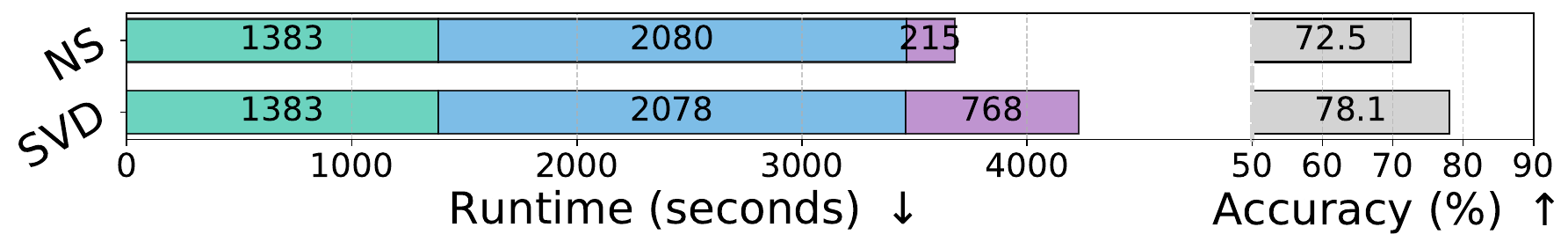}  \\
        \footnotesize{(a) OPT-\textit{1.3B} fine-tuning} & 
        \footnotesize{(b) OPT-\textit{13B} fine-tuning} \\   
    \end{tabular}  
    \vspace*{-1mm}
    \caption{\small{Runtime and accuracy comparison of ZO-Muon using NS and SVD for fine-tuning OPT-1.3B/13B on SST-2 under the same query budget (5,000).}}
    \label{fig:runtime-dissect}
\end{figure}

\noindent \textbf{SVD vs. Newton--Schulz in ZO-Muon.}
Thanks to the subspace projection in ZO-Muon \eqref{eq:ZO_Muon}, the $\mathrm{msign}$ operation, previously the main computational bottleneck in Muon, can now be computed via exact SVD as in \eqref{eq:matrix_sign}, instead of using NS iterations \citep{jordan2024muon}.
Thus, \textbf{Fig.~\ref{fig:runtime-dissect}} compares the runtime and accuracy of ZO-Muon using NS (the default choice) vs. SVD when fine-tuning OPT-1.3B/13B on SST-2. The runtime is further decomposed into three components: sampling the projection matrix $\mathbf{P}$, forward passes for ZO gradient estimation, and the $\mathrm{msign}$ computation.
As shown, using SVD for exact $\mathrm{msign}$ computation indeed improves accuracy over the NS-based implementation. However, SVD incurs approximately $3\times$ higher runtime than NS for the matrix-sign operation. This substantial efficiency gap is the primary reason we adopt NS in ZO-Muon.

       

\noindent \textbf{Additional results.}
We report additional validations in \textbf{Appendix\,\ref{appendix:additional-results}}. 
\textbf{Fig.\,\ref{fig:loss-curve-appendix}} shows the fast convergence of ZO-Muon across multiple datasets, complementing Fig.\,\ref{fig:motivating-figure}. \textbf{Tab.\,\ref{tab:gemma-table}} reports results on Gemma2-2B \citep{team2024gemma}, where ZO-Muon again consistently outperforms other ZO methods, in line with Tab.\,\ref{tab:main-table}. \textbf{Fig.\,\ref{fig:matched-rank}} reports the performance of ZO-Muon and the baselines with matched ranks. We also study the resampling frequency of the projection matrix $\mathbf P$ and show that lazy resampling stabilizes ZO training in \textbf{Fig.\,\ref{fig:sample-interval}}.

\section{Conclusion}
\label{sec: conclusion}

In this paper, we propose ZO-Muon, a novel ZO optimization method that bridges the gap between memory efficiency and convergence speed in deep model training. By combining subspace optimization with spectral gradient orthogonalization, ZO-Muon resolves the fundamental variance–efficiency tension in standard ZO methods. Extensive experiments on LLMs and ViTs show that ZO-Muon consistently outperforms state-of-the-art baselines in both performance and efficiency.
We refer readers to Appendix~\ref{appendix:limitations}-\ref{appendix:llm-usage} for discussions on limitations, ethics statement, and LLM usage.
\section*{Acknowledgment}
\label{sec: acknowledgment}
The work of Y. Lang, C. Wang, Y. Zhang, and S. Liu was supported in part by the
National Science Foundation (NSF) CAREER Award IIS-2338068. Z. Zhang was supported by NSF CCF-2419889. 
The authors declare no competing financial interests related to this work beyond the funding disclosed above.


\bibliographystyle{unsrtnat}
\bibliography{refs/MU_SLiu, refs/MU,refs/ZO_SLiu,refs/muon_opt,refs/yicheng}


\appendix
\clearpage
\onecolumn
\appendix
\setcounter{section}{0}
\section*{{Appendix}}

\setcounter{section}{0}
\setcounter{figure}{0}
\makeatletter 
\renewcommand{\thefigure}{A\arabic{figure}}
\renewcommand{\theHfigure}{A\arabic{figure}}
\renewcommand{\thetable}{A\arabic{table}}
\renewcommand{\theHtable}{A\arabic{table}}

\makeatother
\setcounter{table}{0}
\setcounter{equation}{0}
\renewcommand{\theequation}{A\arabic{equation}}

\section{Detailed Related Work}
\label{appendix:related_work}

\noindent \textbf{ZO optimization and applications in deep models.}
ZO optimization \citep{ghadimi2013stochastic,duchi2015optimal,nesterov2017random,balasubramanian2022zeroth,liu2020primer,song2026zeroth} estimates gradients via finite differences of function values without computing FO derivatives. As detailed in Sec.\,\ref{sec: intro} and thus omitted here for brevity, a large body of foundational work on ZO optimization has focused on reducing estimation variance to improve accuracy and the trade-off with query efficiency. 
%
Empirically, the effectiveness of ZO methods has been demonstrated in solving small-scale problems, such as input-level black-box adversarial attack and defense \citep{chen2017zoo,tu2019autozoom,chen2019zo,cheng2018queryefficient,ilyas2018black,liu2020min,zhang2022how}, input-level prompt design for foundation models \citep{sun2022black,hu2024localized,zhang2025leveraging}, and resource-limited on-device learning \citep{gu2021efficient,zhang2025foresight,tan2025perturbation}. 
More recently, ZO optimization has expanded to large-scale deep model training, including black-box or non-differentiable settings \citep{chen2023deepzero,sawada2025natural} and memory-efficient fine-tuning \citep{zhang2024revisiting,malladi2023finetuning,tan2025harmony,liu2025sparse}. Focusing on the latter, to the best of our knowledge, this is the first work to demonstrate a \emph{win–win} improvement in both accuracy and efficiency via subspace gradient orthogonalization.

\noindent \textbf{Subspace optimization for efficient training.}
Since the weights or gradients in deep model training often exhibit low-rank structure \citep{kaushik2025universal,aghajanyan2021intrinsic,hu2022lora,zhao2024galore,hao2024flora}, 
a growing body of work has explored training in low-dimensional subspaces to improve memory efficiency \citep{he2024subspace,liang2024memory,zhang2025i3s,refael2025sumo}. Notably, LoRA-style methods \citep{hu2022lora,dettmers2023qlora,wang2025milora,lion2025polar,liu2026flexlora} directly constrain trainable parameters to low-rank matrices, but can be less effective in some settings due to the limited expressiveness of low-rank parameterizations \citep{hao2024flora} and insufficient directional diversity \citep{lion2025polar}.
%
Subspace perspectives have been explored in ZO optimization. Prior work \citep{kozak2023zeroth,nozawa2025zeroth} develops theoretical subspace formulations but does not address practical large-scale fine-tuning, while more practical methods like LOZO \citep{chen2025enhancing} and SubZero \citep{yu2025zeroth} use low-dimensional Gaussian matrices for perturbation and updates. As we show later, our method clearly outperforms these approaches by  exploiting the variance-reduction benefits of subspace ZO optimization in the spectral domain.

\noindent \textbf{Emerging gradient-orthogonalized methods for deep model training.}
Recent years have witnessed growing interest in {gradient-orthogonalized} optimization methods for training large-scale deep models, particularly foundation models, driven by the need to exploit the matrix-structured and spectral properties of gradients that are poorly captured by conventional vector-based optimizers.
The Muon optimizer~\citep{jordan2024muon} and its subsequent variants~\citep{liu2025muon,ma2025swan,riabinin2025gluon,khaled2025muonbp,he2025low,ahn2025dion,amsel2025polar,gong2026aro} embody this shift by performing \emph{gradient orthogonalization} (\textbf{GO}), which projects gradient updates onto a spectrally constrained set via matrix whitening and can be interpreted as a steepest-descent step under a spectral-norm constraint~\citep{bernstein2024old}. This spectral treatment leads to faster convergence and reduced optimizer-state memory compared to Adam, enabling Muon to scale to foundation-model training, most notably in Moonlight~\citep{liu2025muon}, a 3B/16B Mixture-of-Experts (MoE) model. 
Despite this progress, GO has remained largely confined to \emph{FO pre-training}, and extending it to ZO settings is nontrivial due to the lower quality of ZO gradient estimates. JAGUAR-Muon \citep{petrov2025leveraging} attempts to apply GO to full-space ZO momentum, but incurs additional memory and runtime overhead over MeZO \citep{malladi2023finetuning} and shows very limited performance gain.
This is likely because applying GO to noisy, high-dimensional ZO estimates amplifies non-informative spectral components, a limitation also confirmed by our experiments. By contrast, we show that GO becomes effective for ZO optimization \emph{only when coupled with a subspace formulation}, \textit{i.e.}, \emph{subspace gradient orthogonalization}.

\section{Intrinsic Rank Measurement}
\label{appendix:intrinsic-rank}
To estimate the intrinsic dimensionality of the gradients, we perform Singular Value Decomposition (SVD) on the gradient matrix $\mathbf{G}$. Let $\{\sigma_i\}_{i=1}^n$ denote the singular values of $\mathbf{G}$ sorted in descending order. We determine the effective rank by identifying the number of principal components $k$ required to capture $99.99\%$ of the gradient energy. This relationship is formulated as finding the cutoff $k$ satisfying:
$\mathcal{E}(k) = \frac{\sum_{i=1}^{k} \sigma_i^2}{\sum_{i=1}^{n} \sigma_i^2} \ge 0.9999$.

\section{Ineffectiveness of Full Space Gradient Orthogonalization}
\label{appendix:ineffectiveness}

\begin{figure*}[htb]
    \centering
    \includegraphics[width=0.5\linewidth]{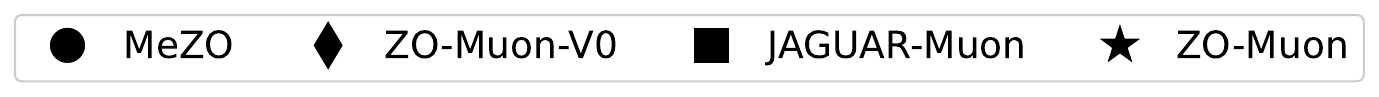}\\
    \includegraphics[width=0.3\linewidth]{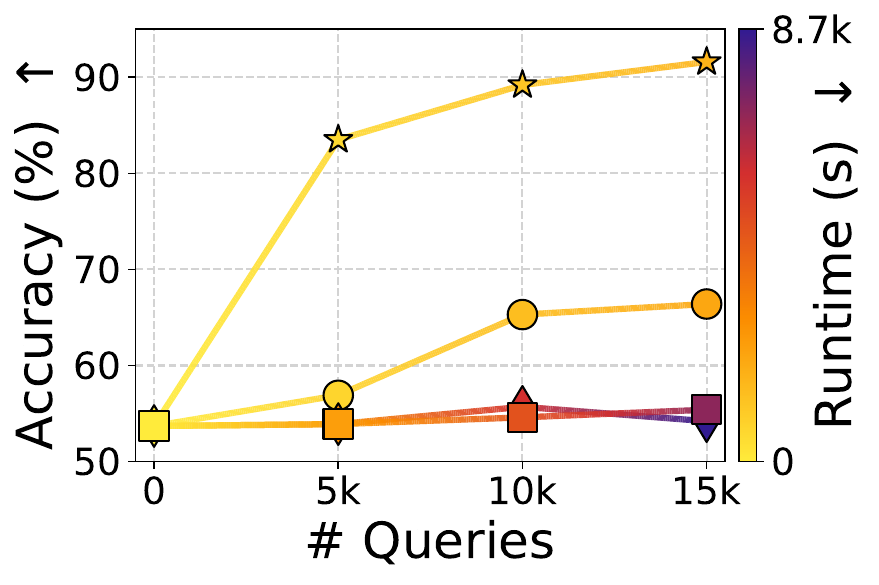}
    \caption{Fine-tune accuracy and runtime of full space gradient orthogonalization methods JAGUAR-Muon and ZO-Muon-V0 compared to MeZO and ZO-Muon, demonstrated by fine-tuning OPT-1.3B on SST-2.
    }
    \label{fig:full-space-appendix}
\end{figure*}

Complementary to \textbf{Fig.~\ref{fig:ZO_Muon_V0}}, we additionally validate the \textit{ineffectiveness} of full space ZO gradient orthogonalization by fine-tuning OPT-1.3B on SST-2, shown in \textbf{Fig.~\ref{fig:full-space-appendix}}. With the same query number, ZO-Muon-V0 and JAGUAR-Muon \citep{petrov2025leveraging} does not outperform MeZO in fine-tune accuracy. Additionally, they consume more runtime compared to MeZO. In contrast, the \textit{subspace} gradient orthogonalization ZO-Muon outperforms the above methods in accuracy and does not sacrifice runtime efficiency.

\section{ZO-Muon Algorithm} 
\label{appendix:zo-muon-formulation}
\textbf{Algorithm~\ref{algo:zo-muon}} details our proposed ZO-Muon method. ZO-Muon treats parameters as matrices and optimizes layer-wise matrices $\{\mathbf{X}_l\}$ where $l$ represents each layer. First, the projection matrices $\{\mathbf{P}_l\}$ are lazily sampled every interval $v$, corresponding to ``\textit{subspace sampling}''. Second, we perform \textit{Subspace-RGE} following \eqref{eq:RGE_subspace}, making $N_q>1$ queries and performing a forward-difference estimate, and average the subspace ZO estimates. Finally, shown in ``full-space weight update'', we follow \eqref{eq:ZO_Muon} to orthogonalize the subspace gradient and project them back to full-space for model update.

\begin{algorithm}[htb]
   \caption{ZO-Muon}
   \label{algo:zo-muon}
   \small
   \begin{algorithmic}
      \STATE {\bfseries Input:} Parameters $\mathbf{X}=\{\mathbf{X}_l\}$, (where $\mathbf{X}_l \in \mathbb{R}^{m_l \times n_l}$), loss function $f(\mathbf{X})$, projection matrices $\mathbf{P}=\{\mathbf{P}_l\}$, learning rate $\eta$, perturbation scale $\mu$, rank parameter $r$, query number $N_q$, projection matrix resample interval $v$, train steps $T$.
      
      \STATE \textit{// Initialize Projection Matrices}
      \STATE Sample column-orthogonal projection matrices $\{\mathbf{P}_l\}$ via QR decomposition.
      \FOR{$t = 0, \dots, T-1$}
         \STATE \textit{// 1. Subspace Sampling (Periodic Update)}
         \IF{$t > 0$ \AND $t \pmod v = 0$}
             \STATE Resample $\{\mathbf{P}_l\}$ via QR decomposition.
         \ENDIF

         \STATE \textit{// 2. Subspace-RGE} \\
         Compute subspace ZO gradient estimate via \eqref{eq:RGE_subspace} with $N_q>1$.
         \STATE \textit{// 3. Full-space weight update} \\
         Perform subspace gradient orthogonalization, then project back to full space to update  via \eqref{eq:ZO_Muon}.
         
      \ENDFOR
      \STATE {\bfseries Return} $\mathbf{X}$ ;

   \end{algorithmic}
\end{algorithm}

\section{Proof of Proposition \ref{prop:low-rank-GO}}
\label{sec:GO_lowrank_appendix}
\textbf{Proof}:
With $\mathrm{msign}(\mathbf G) 
   = \mathbf G (\mathbf G^{\top} \mathbf G)^{-\frac{1}{2}} $ and using $\mathbf{P}^\top \mathbf{P} = \mathbf{I}$, we obtain
\begin{align}
    \mathbf{P}\, \mathrm{msign}(\mathbf{P}^\top \mathbf{G}) 
    & = \mathbf{P}\, \mathbf{P}^\top \mathbf{G} \big( (\mathbf{P}^\top \mathbf{G})^\top (\mathbf{P}^\top \mathbf{G}) \big)^{-1/2} \nonumber \\
    & = \mathbf{P}\, \mathbf{P}^\top \mathbf{G} \big( \mathbf{G}^\top \mathbf{P}\mathbf{P}^\top \mathbf{G} \big)^{-1/2}.
    \label{eq:GO_lowrank_proof1}
\end{align}
Next, the SVD of $\mathbf{G}$ yields
\[
\mathbf{G} = \mathbf{U}_{[:, :k]} \, \mathrm{diag}(\boldsymbol{\sigma}) \, \mathbf{V}_{[:, :k]}^\top,
\]
where $\mathrm{diag}(\boldsymbol{\sigma}) \in \mathbb{R}^{k \times k}$ is the diagonal matrix containing the nonzero singular values of $\mathbf{G}$, and $k = \mathrm{rank}(\mathbf{G})$.

Since $\mathbf{P} = \mathbf{U}_{[:, :k]}$ and $\mathbf{P}^\top \mathbf{P} = \mathbf{I}$, we have
\begin{align}
    \mathbf{P}\mathbf{P}^\top \mathbf{G}
    & = \mathbf{U}_{[:, :k]} \mathbf{U}_{[:, :k]}^\top \mathbf{U}_{[:, :k]} \, \mathrm{diag}(\boldsymbol{\sigma}) \, \mathbf{V}_{[:, :k]}^\top \nonumber \\
    & = \mathbf{U}_{[:, :k]} \, \mathrm{diag}(\boldsymbol{\sigma}) \, \mathbf{V}_{[:, :k]}^\top \nonumber \\
    & = \mathbf{G}.
    \label{eq:GO_lowrank_proof2}
\end{align}
That is, $\mathbf{P}\mathbf{P}^\top$ acts as the identity on the column space of $\mathbf{G}$. Substituting~\eqref{eq:GO_lowrank_proof2} into~\eqref{eq:GO_lowrank_proof1}, we finally obtain
$\mathbf P  \mathrm{msign} (\mathbf{P}^\top  \mathbf G) =  \mathrm{msign}(\mathbf G)$. This completes the proof.

\section{Convergence Analysis}
\label{appendix:convergence}
 To simplify the analysis without loss of generality, for the matrix parameter $\mathbf{X} \in \mathbb{R}^{m \times n}$, we assume a static column-orthogonal projection matrix $\mathbf{P} \in \mathbb{R}^{m \times r}$ ($\mathbf{P}^\top \mathbf{P} = \mathbf{I}$) with $r \ll \min \{ m, n \}$. The exact subspace gradient is denoted as $\nabla_{\mathbf{Z}} f(\mathbf{X}) = \mathbf{P}^\top \nabla f(\mathbf{X})$. 
ZO-Muon utilizes randomized gradient estimate (RGE) to compute the subspace gradient using $N_{\mathrm{q}}$ queries:
\begin{align}
    \hat{\nabla}_{\mathbf{Z}} f(\mathbf{X}) = \frac{1}{N_{\mathrm{q}}} \sum_{i=1}^{N_{\mathrm{q}}} \frac{f(\mathbf{X} + \mu \mathbf{P} \boldsymbol{\Psi}_i) - f(\mathbf{X})}{\mu}\,\boldsymbol{\Psi}_i.
\end{align}
The above yields an unbiased gradient estimate of the \textit{smoothed} objective
$f_\mu(\mathbf{X}) = \mathbb{E}_{\boldsymbol{\Psi}\sim \mathcal N (\mathbf 0, \mathbf I)}[f(\mathbf{X} + \mu \mathbf{P}\boldsymbol{\Psi})]$,
where $\mu > 0$ is the smoothing parameter. That is, 
\begin{align}
    \mathbb{E}_{\boldsymbol{\Psi}}[\hat{\nabla}_{\mathbf{Z}} f(\mathbf{X})] = \nabla_{\mathbf{Z}} f_\mu(\mathbf{X}).
\end{align}

The above unbiasedness property facilitates convergence analysis of ZO-Muon in expectation. Thus, we focus on the convergence rate of its smoothed variant, which, under the unbiasedness condition, induces the following Muon update:
\begin{align}
    \mathbf{X}_{t+1} = \mathbf{X}_t - \eta \mathbf{P} \mathrm{msign}(\nabla_{\mathbf{Z}} f_\mu(\mathbf{X}_t)),
    \label{eq:smoothed_zo_muon}
\end{align}
where $t$ represents the optimization step.

\subsection{Assumptions}

We make the following assumptions standard for the convergence analysis of non-convex optimization.

\vspace{2mm}
\noindent\textbf{Assumption 1 ($L$-smoothness).}
The objective function $f$ is smooth with constant $L$, meaning that for any $\mathbf{X}, \mathbf{Y}$, it satisfies:
\begin{align}
    f(\mathbf{Y}) \leq f(\mathbf{X}) + \langle \nabla f(\mathbf{X}), \mathbf{Y}-\mathbf{X} \rangle + \frac{L}{2} \|\mathbf{Y}-\mathbf{X}\|_{\mathrm{F}}^2,
\end{align}
where $\| \cdot \|_{\mathrm{F}}$ denotes the Frobenius norm. 


\vspace{2mm}
\noindent\textbf{Assumption 2 (Subspace alignment).}
The projection matrix $\mathbf{P} \in \mathbb{R}^{m \times r}$ captures a fraction $\delta \in (0, 1]$ of the full gradient. Specifically, $\|\nabla_{\mathbf{Z}} f(\mathbf{X})\|_{\mathrm{F}} \geq \delta \|\nabla f(\mathbf{X})\|_{\mathrm{F}}$.

\subsection{Supporting Lemmas}

\begin{mylemma}[Theorem 1 in \citep{nesterov2017random}]
\label{lemma:zo_bias}
Under Assumption 1, the expected smoothing bias of the ZO subspace gradient estimator is bounded by:
\begin{align}
    \| \nabla_{\mathbf{Z}} f(\mathbf{X}_t) - \nabla_{\mathbf{Z}} f_\mu(\mathbf{X}_t) \|_{\mathrm{F}} \leq \mu L \sqrt{rn},
    \label{eq:bias_smoothing}
\end{align}
where $t$ denotes the optimization step, $\mu$ is the smoothing parameter, $L$ is the Lipschitz constant defined in Assumption 1, $r$ is the projection dimension, and $n$ is the column dimension of $\mathbf{X}$.
\end{mylemma}

\begin{mylemma}
\label{lemma:msign_bound}
For the parameter update driven by \eqref{eq:smoothed_zo_muon}, the inner product between the subspace gradient and the gradient-orthogonalized direction satisfies:
\begin{align}
    \langle \nabla_{\mathbf{Z}} f(\mathbf{X}_t), \mathrm{msign}(\nabla_{\mathbf{Z}} f_\mu(\mathbf{X}_t)) \rangle \geq \|\nabla_{\mathbf{Z}} f(\mathbf{X}_t)\|_* - 2\sqrt{r} \|\nabla_{\mathbf{Z}} f(\mathbf{X}_t) - \nabla_{\mathbf{Z}} f_\mu(\mathbf{X}_t)\|_{\mathrm{F}},
    \label{eq:lb_msign}
\end{align}
where $\| \cdot \|_*$ represents the nuclear norm of a matrix.
\end{mylemma}
\noindent \textit{Proof.} 
Because $\mathbf{P} \in \mathbb{R}^{m \times r}$ is column-orthogonal ($\mathbf{P}^\top \mathbf{P} = \mathbf{I}$), the Frobenius norm is invariant under projection, meaning $\|\mathbf{P} \mathrm{msign}(\nabla_{\mathbf{Z}} f_\mu(\mathbf{X}_t))\|_{\mathrm{F}}^2 = \|\mathrm{msign}(\nabla_{\mathbf{Z}} f_\mu(\mathbf{X}_t))\|_{\mathrm{F}}^2$. Furthermore, because the subspace gradient $\nabla_{\mathbf{Z}} f_\mu(\mathbf{X}_t) \in \mathbb{R}^{r \times n}$ has a maximum rank of $r$, and the $\mathrm{msign}$ operator sets all non-zero singular values exactly to $1$, the squared Frobenius norm, which is the sum of squared singular values, is bounded by $r$. Thus, $\|\mathrm{msign}(\nabla_{\mathbf{Z}} f_\mu(\mathbf{X}_t))\|_{\mathrm{F}}^2 \leq r$. Therefore:
\begin{align*}
    \langle \nabla_{\mathbf{Z}} f(\mathbf{X}_t), \mathrm{msign}(\nabla_{\mathbf{Z}} f_\mu(\mathbf{X}_t)) \rangle &= \langle \nabla_{\mathbf{Z}} f_\mu(\mathbf{X}_t), \mathrm{msign}(\nabla_{\mathbf{Z}} f_\mu(\mathbf{X}_t)) \rangle \\
    &\quad + \langle \nabla_{\mathbf{Z}} f(\mathbf{X}_t) - \nabla_{\mathbf{Z}} f_\mu(\mathbf{X}_t), \mathrm{msign}(\nabla_{\mathbf{Z}} f_\mu(\mathbf{X}_t)) \rangle \\
    &\geq \|\nabla_{\mathbf{Z}} f_\mu(\mathbf{X}_t)\|_* - \|\nabla_{\mathbf{Z}} f(\mathbf{X}_t) - \nabla_{\mathbf{Z}} f_\mu(\mathbf{X}_t)\|_{\mathrm{F}} \|\mathrm{msign}(\nabla_{\mathbf{Z}} f_\mu(\mathbf{X}_t))\|_{\mathrm{F}} \\
    &\geq \|\nabla_{\mathbf{Z}} f_\mu(\mathbf{X}_t)\|_* - \sqrt{r} \|\nabla_{\mathbf{Z}} f(\mathbf{X}_t) - \nabla_{\mathbf{Z}} f_\mu(\mathbf{X}_t)\|_{\mathrm{F}}.
\end{align*}
Applying the reverse triangle inequality for nuclear norms ($\| \nabla_{\mathbf{Z}} f(\mathbf{X}_t) \|_* \leq \|\nabla_{\mathbf{Z}} f_\mu(\mathbf{X}_t)\|_* + \| \nabla_{\mathbf{Z}} f(\mathbf{X}_t) - \nabla_{\mathbf{Z}} f_\mu(\mathbf{X}_t) \|_*$) and leveraging the fact that $\|\mathbf{A}\|_* \leq \sqrt{r}\|\mathbf{A}\|_{\mathrm{F}}$ for rank $r$ matrices, we obtain the stated lower bound \eqref{eq:lb_msign}.
\hfill $\square$

\subsection{Main Convergence Result}
\begin{mythr}
\label{thm:zo_muon_convergence}
Suppose Assumptions 1 and 2 hold. Let the learning rate be chosen as $\eta = \sqrt{\frac{2\Delta}{L r T}}$ where $\Delta = f(\mathbf{X}_0) - f^*$ is the initial suboptimality, and $f^*$ is the optimal value. Let the smoothing parameter be set as $\mu = \frac{1}{L\sqrt{rnT}}$. Then, the iterates generated by the smoothed ZO-Muon \eqref{eq:smoothed_zo_muon} converge to a stationary point at a rate of:
\begin{align}
    \frac{1}{T} \sum_{t=0}^{T-1} \|\nabla f(\mathbf{X}_t)\|_{\mathrm{F}} \leq \mathcal{O} \left( \frac{1}{\delta} \sqrt{\frac{r}{T}} \right),
\end{align}
where $T$ denotes the number of iterations, $\delta$ is the subspace alignment ratio, and $r$ is the projection dimension.
\end{mythr}

\noindent \textit{Proof.} 
By the $L$-smoothness assumption (Assumption 1), the parameter update \eqref{eq:smoothed_zo_muon} yields:
\begin{align*}
    f(\mathbf{X}_{t+1}) &\leq f(\mathbf{X}_t) - \eta \langle \nabla f(\mathbf{X}_t), \mathbf{P} \mathrm{msign}(\nabla_{\mathbf{Z}} f_\mu(\mathbf{X}_t)) \rangle + \frac{L \eta^2}{2} \|\mathbf{P} \mathrm{msign}(\nabla_{\mathbf{Z}} f_\mu(\mathbf{X}_t))\|_{\mathrm{F}}^2 \\
    &\leq f(\mathbf{X}_t) - \eta \langle \nabla_{\mathbf{Z}} f(\mathbf{X}_t), \mathrm{msign}(\nabla_{\mathbf{Z}} f_\mu(\mathbf{X}_t)) \rangle + \frac{L \eta^2 r}{2},
\end{align*}
where the last inequality holds due to $\|\mathbf{P} \mathrm{msign}(\nabla_{\mathbf{Z}} f_\mu(\mathbf{X}_t))\|_{\mathrm{F}}^2 = \|\mathrm{msign}(\nabla_{\mathbf{Z}} f_\mu(\mathbf{X}_t))\|_{\mathrm{F}}^2 \leq r$ as used in the proof of Lemma\,\ref{lemma:msign_bound}.

Applying the inner product bound from Lemma \ref{lemma:msign_bound}, we obtain
\begin{align*}
    f(\mathbf{X}_{t+1}) \leq f(\mathbf{X}_t) - \eta \|\nabla_{\mathbf{Z}} f(\mathbf{X}_t)\|_* + 2\eta \sqrt{r} \|\nabla_{\mathbf{Z}} f(\mathbf{X}_t) - \nabla_{\mathbf{Z}} f_\mu(\mathbf{X}_t)\|_{\mathrm{F}} + \frac{L \eta^2 r}{2}.
\end{align*}

Rearranging the terms to isolate the gradient norm and summing over $T$ iterations, we obtain a telescoping sum for the objective function values:
\begin{align*}
    \eta \sum_{t=0}^{T-1} \|\nabla_{\mathbf{Z}} f(\mathbf{X}_t)\|_* &\leq \sum_{t=0}^{T-1} \left[ f(\mathbf{X}_t) - f(\mathbf{X}_{t+1}) \right] + 2\eta \sqrt{r} \sum_{t=0}^{T-1} \|\nabla_{\mathbf{Z}} f(\mathbf{X}_t) - \nabla_{\mathbf{Z}} f_\mu(\mathbf{X}_t)\|_{\mathrm{F}} + \frac{T L \eta^2 r}{2} \\
    &= f(\mathbf{X}_0) - f(\mathbf{X}_T) + 2\eta \sqrt{r} \sum_{t=0}^{T-1} \|\nabla_{\mathbf{Z}} f(\mathbf{X}_t) - \nabla_{\mathbf{Z}} f_\mu(\mathbf{X}_t)\|_{\mathrm{F}} + \frac{T L \eta^2 r}{2}.
\end{align*}

By the definition of $f^*$, we obtain $f(\mathbf{X}_0) - f(\mathbf{X}_T) \leq f(\mathbf{X}_0) - f^* \Def \Delta$. Substituting this yields:
\begin{align*}
    \eta \sum_{t=0}^{T-1} \|\nabla_{\mathbf{Z}} f(\mathbf{X}_t)\|_* \leq \Delta + 2\eta \sqrt{r} \sum_{t=0}^{T-1} \|\nabla_{\mathbf{Z}} f(\mathbf{X}_t) - \nabla_{\mathbf{Z}} f_\mu(\mathbf{X}_t)\|_{\mathrm{F}} + \frac{T L \eta^2 r}{2}.
\end{align*}

Further, dividing the entire equation by $\eta T$, we obtain:
\begin{align*}
    \frac{1}{T} \sum_{t=0}^{T-1} \|\nabla_{\mathbf{Z}} f(\mathbf{X}_t)\|_* \leq \frac{\Delta}{\eta T} + \frac{L \eta r}{2} + \frac{2\sqrt{r}}{T} \sum_{t=0}^{T-1} \|\nabla_{\mathbf{Z}} f(\mathbf{X}_t) - \nabla_{\mathbf{Z}} f_\mu(\mathbf{X}_t)\|_{\mathrm{F}}.
\end{align*}

Substituting \eqref{eq:bias_smoothing} from Lemma \ref{lemma:zo_bias}, we obtain:
\begin{align*}
    \frac{1}{T} \sum_{t=0}^{T-1} \|\nabla_{\mathbf{Z}} f(\mathbf{X}_t)\|_* \leq \frac{\Delta}{\eta T} + \frac{L \eta r}{2} + 2\mu L r \sqrt{n}.
\end{align*}
Since the nuclear norm is bounded below by the Frobenius norm ($\|\mathbf{A}\|_* \geq \|\mathbf{A}\|_{\mathrm{F}}$), we apply the subspace alignment condition from Assumption 2 ($\|\nabla_{\mathbf{Z}} f(\mathbf{X}_t)\|_{\mathrm{F}} \geq \delta \|\nabla f(\mathbf{X}_t)\|_{\mathrm{F}}$) to convert this to the full gradient:
\begin{align} \label{eq:non_asymptotic}
    \frac{1}{T} \sum_{t=0}^{T-1} \|\nabla f(\mathbf{X}_t)\|_{\mathrm{F}} \leq \frac{1}{\delta} \left[ \frac{\Delta}{\eta T} + \frac{L \eta r}{2} + 2 \mu L r \sqrt{n} \right].
\end{align}
Equation (\ref{eq:non_asymptotic}) represents the generalized non-asymptotic convergence bound. To obtain the asymptotic rate, we set $\eta = \sqrt{\frac{2\Delta}{L r T}}$ and $\mu = \frac{1}{L\sqrt{rnT}}$. 

Evaluating the learning rate terms:
\begin{align*}
    \frac{\Delta}{\eta T} + \frac{L \eta r}{2} &= \sqrt{\frac{\Delta L r}{2T}} + \sqrt{\frac{\Delta L r}{2T}} = \sqrt{\frac{2\Delta L r}{T}} = \mathcal{O}\left(\sqrt{\frac{r}{T}}\right).
\end{align*}
Evaluating the smoothing bias term:
\begin{align*}
    2 \mu L r \sqrt{n} &= 2 \left( \frac{1}{L\sqrt{rnT}} \right) L r \sqrt{n} = 2 \sqrt{\frac{r}{T}} = \mathcal{O}\left(\sqrt{\frac{r}{T}}\right).
\end{align*}
Combining these simplified components confirms the final asymptotic rate is independent of $n$:
\begin{align*}
    \frac{1}{T} \sum_{t=0}^{T-1} \|\nabla f(\mathbf{X}_t)\|_{\mathrm{F}} \leq \mathcal{O} \left( \frac{1}{\delta} \sqrt{\frac{r}{T}} \right).
\end{align*}
The proof is now complete. 
\hfill $\square$

\section{Hyperparameter Study on $N_q$ in ZO-Muon}
\label{appendix:N_q-validation}
In \textbf{Fig.\,\ref{fig:Nq_validation}}, we compare ZO-Muon with Subspace-MeZO under different query number $N_q$, plotting performance against the total number of queries used during training. Shown in Fig.\,\ref{fig:Nq_validation}(a), there exists a multi-query paradox: Subspace-MeZO with $N_q>1$ performs worse than $N_q=1$. On the contrary, ZO-Muon does not suffer from this issue (\textit{i.e.}, additional queries do not lead to diminishing optimization returns). Moreover, as shown in Fig.\,\ref{fig:Nq_validation}(b), although ZO-Muon uses $N_q = 4$, this does not hurt query efficiency in practice compared to Subspace MeZO with $N_q = 1$, since ZO-Muon converges in significantly fewer iterations.

\begin{figure}[htb]
    \centering
    \includegraphics[width=0.6\columnwidth]{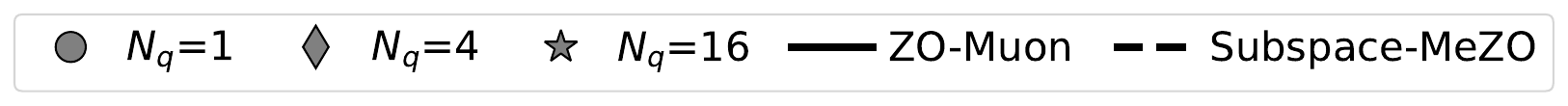}
    \vspace*{-2mm}
    \begin{tabular}{c   c}
        \includegraphics[width=0.3\columnwidth]{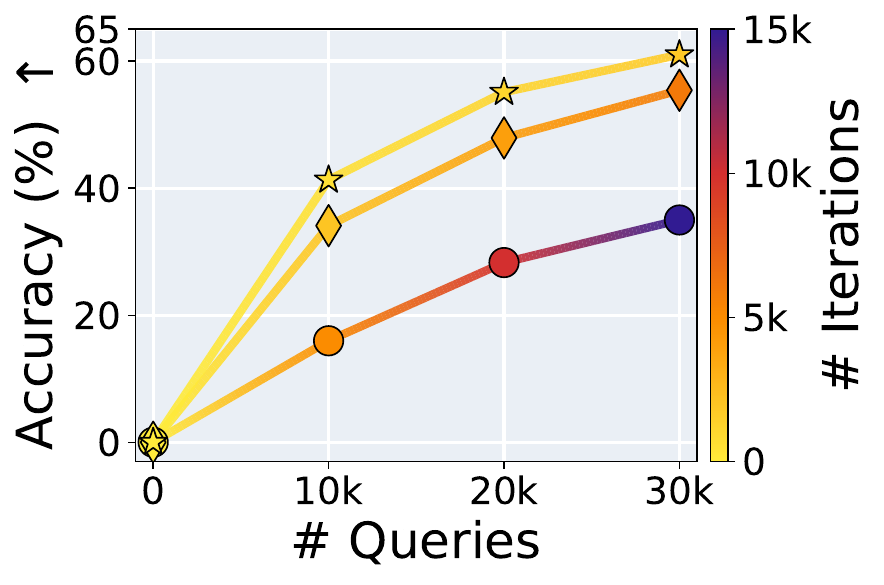} &
        \includegraphics[width=0.3\columnwidth]{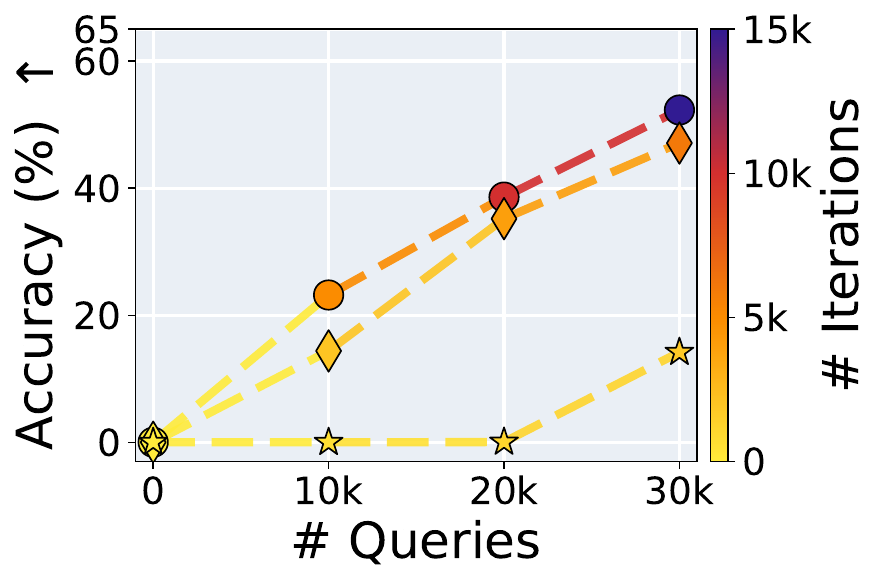} 
        \vspace*{-1mm}
        \\
       \footnotesize{(a) ZO-Muon} &   
        \footnotesize{(b) Subspace-MeZO}\\
    \end{tabular}
    \vspace*{2mm}
    \caption{Fine-tuning accuracy vs. query cost for (a) ZO-Muon and (b) Subspace-MeZO when fine-tuning ViT-B on CIFAR-100, with marker color indicating the used iteration number. The per-iteration query budget is set to $N_q \in \{1, 4, 16\}$.
\textit{(a)} ZO-Muon benefits from increasing $N_q$, with a notable performance gain when moving from $N_q = 1$ to $N_q = 4$.
\textit{(b)} Under the same setting, Subspace-MeZO exhibits a multi-query paradox, where performance degrades when $N_q > 1$.
    }
    \vspace*{-2mm}
    \label{fig:Nq_validation}
\end{figure}

\section{Detailed Experiment Setups}
\label{appendix:experiment-details}

\paragraph{Models and datasets.} 
For LLM experiments, we evaluate Gemma2-2B \citep{team2024gemma}, LlaMA3-8B \citep{grattafiori2024LlaMA} and OPT-13B \citep{zhang2022opt} on 
the SuperGLUE benchmark \citep{wang2019superglue}, including the tasks SST-2 \citep{socher2013recursive}, RTE \citep{dagan2005pascal}, CB \citep{de2019commitmentbank}, BoolQ \citep{clark2019boolq}, WiC \citep{pilehvar2019wic} and SQuAD \citep{rajpurkar2016squad}. Following previous literature \citep{malladi2023finetuning,chen2025enhancing,zhao2025secondorder,yu2025zeroth}, we randomly sample 1,000 train samples and 1,000 test samples. We adopt the same prompts as MeZO \citep{malladi2023finetuning}.

For ViT experiments, we evaluate ViT-B/16 and ViT-L/16 \citep{dosovitskiy2021an} fine-tuned on the CIFAR datasets \citep{krizhevsky2009learning}. Data processing and evaluations are based on the well-established ViT pytorch code base\footnote{\url{https://github.com/jeonsworld/ViT-pytorch}}.

\paragraph{Implementation details.} Our ZO-Muon, as well as baselines LOZO and SubZero treat model parameters as \textit{matrices} instead of vectors. 
Following Muon \citep{jordan2024muon}, we flatten the last three dimensions of 4D convolutional parameters and treat them as 2D matrices for optimization. Vector parameters, as well as the input and output layers, are optimized with MeZO.

\paragraph{Hyperparameters.}
We adopt hyperparameter settings consistent with existing literature to ensure rigorous benchmarking. 
For LLMs, we use a batch size of 16 and train for 20,000 steps, utilizing a constant learning rate without weight decay \citep{malladi2023finetuning, chen2025enhancing, yu2025zeroth}. The perturbation step size is fixed at $\mu=10^{-3}$ based on its demonstrated robustness across tasks \citep{malladi2023finetuning, chen2025enhancing}. For ViT experiments, we set the batch size to 64 (CIFAR-10) and 256 (CIFAR-100). The necessity for increased batch sizes on CIFAR-100 is demonstrated in Fig.~\ref{fig:vit-bs} and Sec.~\ref{sec: experiment}. We also fix the perturbation stepsize $\mu=10^{-3}$.
For algorithmic hyperparameters (e.g., sparsity for S-MeZO; rank $r$ and interval $v$ for subspace-based methods like LOZO and SubZero), we strictly follow the original baseline implementations. For LoRA, we follow previous works \citep{malladi2023finetuning, chen2025enhancing} and set rank $r=8$ and the scaling factor $\alpha=16$.
Detailed settings for both ZO and FO methods are listed in \textbf{Tab.~\ref{tab:hyperparam}}.

\paragraph{Computing resources.} Experiments are primarily conducted on A6000 GPUs with 48GB memory. OPT-13B full Adam fine-tuning is run on H100 GPUs with 100GB memory since A6000 does not have enough memory under our setting.

\begin{table}[htb]
\centering
\caption{Hyperparameter settings of ZO and FO methods used for fine-tuning LLMs and ViTs.}
\label{tab:hyperparam}
\resizebox{14cm}{!}{%
    \small 
    \begin{tabular}{l l c c}
    \toprule
    \toprule
  \multirow{2}{*}{\textbf{Methods}}   & \multirow{2}{*}{\textbf{Hyperparameters}}  & \multicolumn{2}{c}{\textbf{Values}} \\
    \cmidrule(lr){3-4}
     &  & \textbf{LLM} & \textbf{ViT} \\
    \midrule
    \multirow{2}{*}{MeZO} 
     & Learning rate & \makecell[l]{$\{1\text{e}-7, 1\text{e}-6\}$ or $\{1\text{e}-7, 5\text{e}-7, 1\text{e}-6\}$ for SQuAD} & $\{1\text{e}-5, 1\text{e}-4\}$ \\
     & Train steps & 20k & 20k \\
     \midrule
     \multirow{3}{*}{S-MeZO} 
     & Learning rate & \makecell[l]{$\{1\text{e}-7, 1\text{e}-6\}$ or $\{1\text{e}-7, 5\text{e}-7, 1\text{e}-6\}$ for SQuAD} & $\{1\text{e}-5, 1\text{e}-4\}$ \\
     & Sparsity & $0.8$ & $0.8$ \\
     & Train steps & 20k & 20k \\
     \midrule
     \multirow{2}{*}{HiZOO} 
     & Learning rate & \makecell[l]{$\{1\text{e}-7, 1\text{e}-6\}$ or  $\{1\text{e}-7, 5\text{e}-7, 1\text{e}-6\}$ for SQuAD} & $\{1\text{e}-5, 1\text{e}-4\}$ \\
     & Train steps & 20k & 20k \\
     \midrule
    \multirow{4}{*}{LOZO} 
     & Learning rate & $\{1\text{e}-7, 1\text{e}-6\}$ & $\{1\text{e}-5, 1\text{e}-4\}$ \\
     & Rank ($r$) & $\{2, 4, 8\}$ & $\{2, 4, 8\}$ \\
     & Interval ($v$) & $100$ & $100$ \\
     & Train steps & 20k & 20k \\
    \midrule
    \multirow{4}{*}{SubZero} 
     & Learning rate & $\{1\text{e}-7, 1\text{e}-6\}$ & $\{1\text{e}-5, 1\text{e}-4\}$ \\
     & Rank ($r$) & $\{24,48\}$ & $\{24,48\}$ \\
     & Interval ($v$) & $1000$ & $1000$ \\
     & Train steps & 20k & 20k \\
    \midrule
    \multirow{4}{*}{Subspace-MeZO} 
     & Learning rate & $1\text{e}-5$ & $\{1\text{e}-5, 1\text{e}-4\}$ \\
     & Rank ($r$) & $\{64,128\}$ & $\{64,128\}$ \\
     & Interval ($v$) & $100$ & $100$ \\
     & Train steps & 20k & 20k \\
    \midrule
    \multirow{6}{*}{ZO-Muon} 
     & Learning rate & $1\text{e}-2$ & $1\text{e}-2$ \\
     & Rank ($r$) & $\{64,128\}$ & $\{64,128\}$ \\
     & Interval ($v$) & $100$ & $100$ \\
     & Train steps ($N_q=4$) & 8k & 10k \\
     & Train steps ($N_q=8$) & 4k & \_ \\
     & Train steps ($N_q=16$) & \_ & 2k \\
    \midrule
    \midrule
    \multirow{2}{*}{Full} 
     & Learning rate & $\{5\text{e}-6, 1\text{e}-5\}$ & $\{1\text{e}-3, 5\text{e}-3\}$ \\
     & Epochs  & 10 & 10 \\
     \midrule
     \multirow{4}{*}{LoRA} 
     & Learning rate & $\{5\text{e}-6, 1\text{e}-5\}$ & $\{1\text{e}-3, 5\text{e}-3\}$ \\
     & Epochs  & 10 & 10 \\
     & Rank ($r$) & 8 & 8 \\
     & Alpha ($\alpha$) & 16 & 16 \\
     \midrule
     \multirow{3}{*}{Galore} 
     & Learning rate & $\{5\text{e}-6, 1\text{e}-5\}$ & $\{1\text{e}-3, 5\text{e}-3\}$ \\
     & Epochs  & 10 & 10 \\
     & Rank ($r$) & 128 & 128 \\
    \bottomrule
    \bottomrule
    \end{tabular}%
}
\vspace*{-2mm}
\end{table}

\vspace*{-1mm}
\section{Additional Experiment Results}
\label{appendix:additional-results}
\vspace*{-1mm}



\begin{figure*}[htb]
    \centering
    \includegraphics[width=0.35\linewidth]{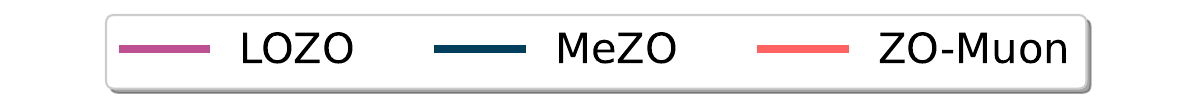}
    \begin{tabular}{c @{\hspace{2mm}} c @{\hspace{2mm}} c @{\hspace{2mm}} c}
        \includegraphics[width=0.24\linewidth]{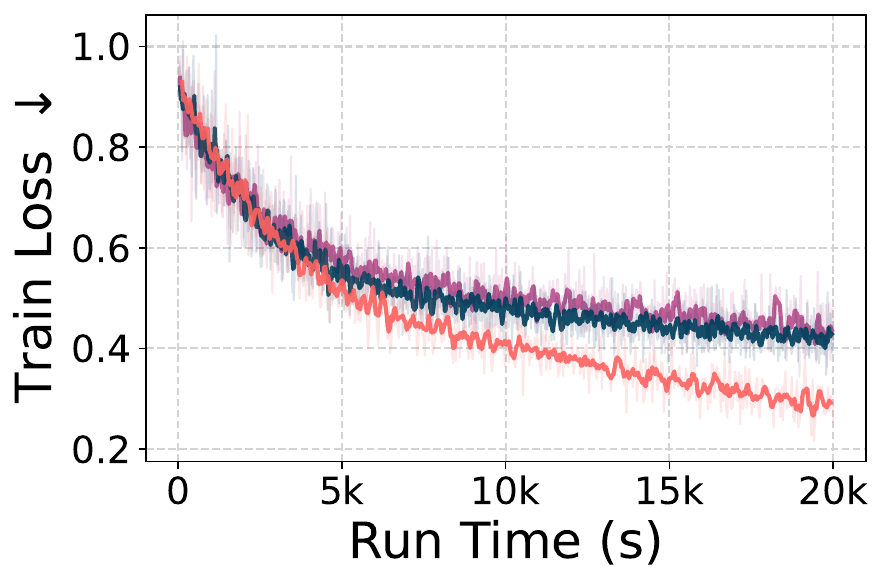} & 
        \includegraphics[width=0.24\linewidth]{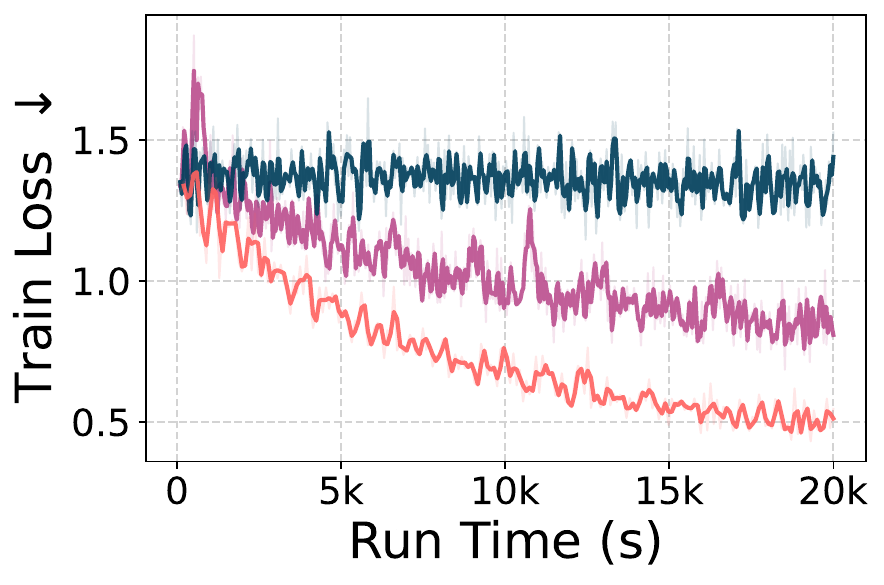} & 
        \includegraphics[width=0.24\linewidth]{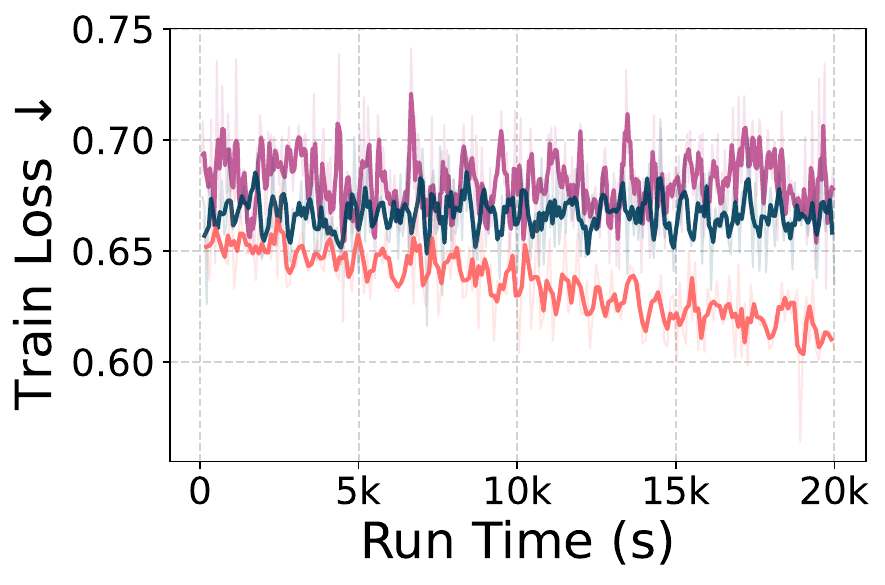} & 
        \includegraphics[width=0.24\linewidth]{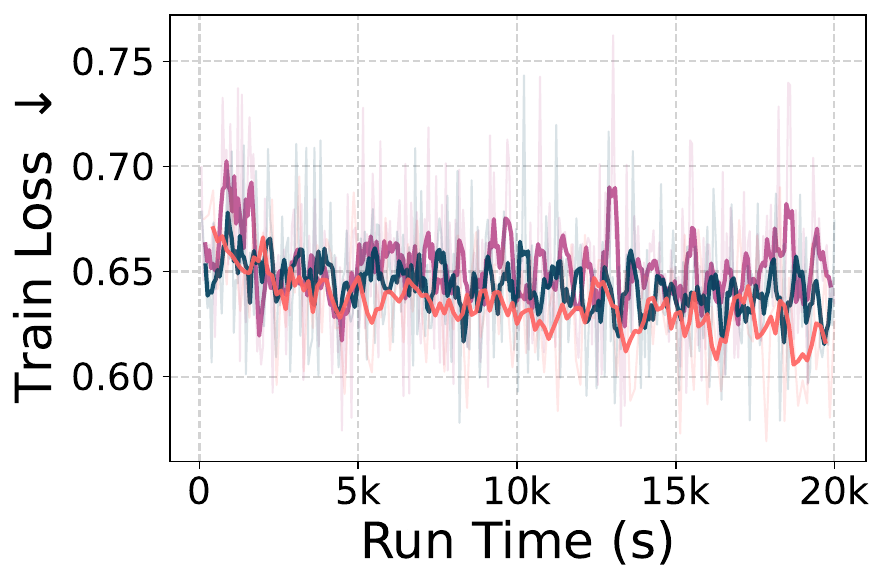} \\
        \footnotesize{(a) SST-2} & 
        \footnotesize{(b) SQuAD} &
        \footnotesize{(c) RTE} &
        \footnotesize{(d) BoolQ} \\
    \end{tabular}
    
    \caption{ZO-Muon achieves faster convergence rate compared to MeZO and LOZO, demonstrated through fine-tuning OPT-13B on SST-2, SQuAD, RTE and BoolQ.
    }
    \label{fig:loss-curve-appendix}
\end{figure*}

\paragraph{Loss curves of LLM fine-tuning.} Complementary to Fig.\,\ref{fig:motivating-figure}, \textbf{Fig.\,\ref{fig:loss-curve-appendix}} shows the training loss curves of ZO-Muon and representative baselines (MeZO and LOZO) for OPT-13B fine-tuning on four downstream tasks.
Note that each method is plotted using its best hyperparameter setting from the search grid (see Appendix~\ref{appendix:experiment-details}). Across all four tasks, ZO-Muon achieves the fastest convergence among the baselines.


 \paragraph{Gemma2-2B results.}
 \textbf{Tab.\,\ref{tab:gemma-table}} presents fine-tuning results on Gemma2-2B. Consistent with Sec.\,\ref{sec:experiment-result}, ZO-Muon achieves the best overall ZO performance and comes closest to FO methods. Notably, it improves over MeZO by $10.2\%$ on RTE and $8.5\%$ on BoolQ.

\begin{table}[htb]
\centering
\small
\caption{Performance of ZO-Muon and baselines on Gemma2-2B under the same settings as Tab.\ref{tab:main-table}.}
\label{tab:gemma-table}
\resizebox{0.6\columnwidth}{!}{%
\begin{tabular}{lcccccc}
\toprule
\toprule
\multirow{2}{*}{\textbf{Method}} & \multicolumn{6}{c}{\textbf{Gemma2-2B Fine-tuning}} \\
\cmidrule(lr){2-7}
 & \textbf{SST-2} & \textbf{RTE} & \textbf{CB} & \textbf{BoolQ} & \textbf{WiC} & \textbf{SQuAD} \\
\midrule
$\orangedot$ Full    & 94.0 & 84.0 & 84.0 & 90.0 & 72.0 & 89.4 \\
$\orangedot$ LoRA    & 91.6 & 57.8 & 55.4 & 75.2 & 56.1 & 86.9 \\
$\orangedot$ Galore    & 94.6 & 79.1 & 66.1 & 80.4 & 64.1 & 87.3 \\
MeZO    & 91.7 & 53.0 & 60.7 & 66.1 & 57.1 & 74.1 \\
S-MeZO  & 92.7 & 53.8 & 53.6 & 63.9 & 55.8 & 76.1 \\
HiZOO   & 92.1 & 53.1 & 60.7 & 66.6 & 55.2 & 73.4 \\
LOZO    & 92.2 & 56.3 & 57.1 & 65.0 & 50.2 & 77.3 \\
SubZero & 91.7 & 56.7 & \textbf{67.8} & 63.5 & 58.3 & 72.1 \\
FZOO & 93.6 & 57.0 & 64.3 & 66.2 & 55.5 & 83.2  \\
\textbf{ZO-Muon} & \textbf{94.2} & \textbf{63.2} & 64.3 & \textbf{74.6} & \textbf{60.3} & \textbf{84.3} \\
\bottomrule
\bottomrule
\end{tabular}%
}
\end{table}

\begin{figure}[htb]
    \centering
    \includegraphics[width=0.4\columnwidth]{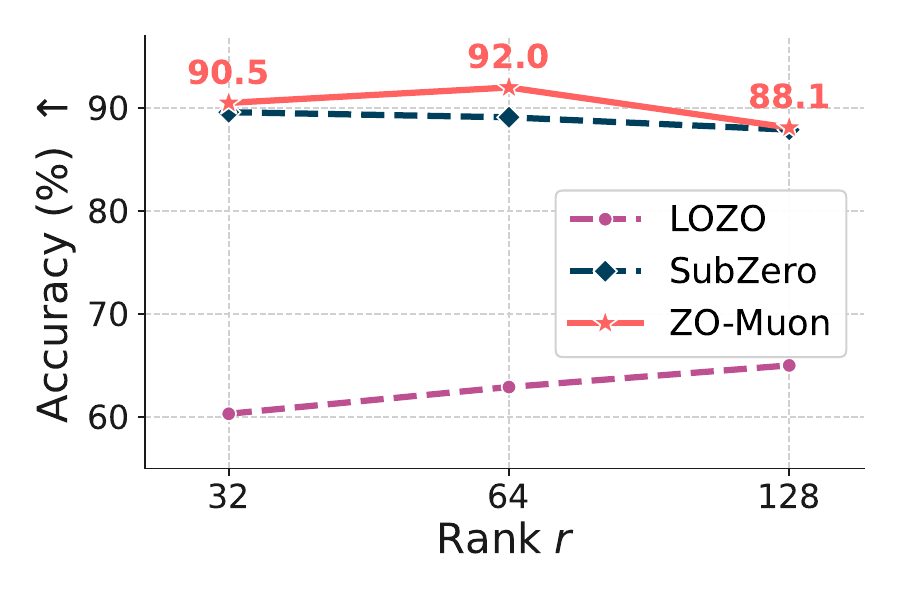} 
    \caption{Performance of ZO-Muon, LOZO, and SubZero with matched query budget and different ranks, demonstrated by fine-tuning OPT-1.3B on SST-2.}
    \label{fig:matched-rank}
\end{figure}

\paragraph{Performance comparison with matched rank.} In the main experiments, LOZO and SubZero use the ranks suggested in their original papers, and all methods share the same overall query budget. To isolate the effect of rank, \textbf{Fig.\,\ref{fig:matched-rank}} compares LOZO, SubZero and ZO-Muon at the same ranks $r \in \{32, 64, 128\}$ with the same budget of 40k queries on OPT-1.3B / SST-2. ZO-Muon outperforms both baselines at every rank. The gain therefore comes from combining subspace optimization with Muon-style gradient orthogonalization, not from a larger rank. This agrees with the analysis of Subspace RGE in Fig.\,\ref{fig:rank}, where a larger rank increases the variance of the gradient estimate. 

\begin{figure}[htb]
    \centering
    \includegraphics[width=0.4\columnwidth]{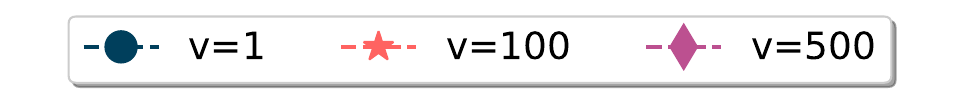} \\
    \begin{tabular}{cc}
        \includegraphics[width=0.3\columnwidth]{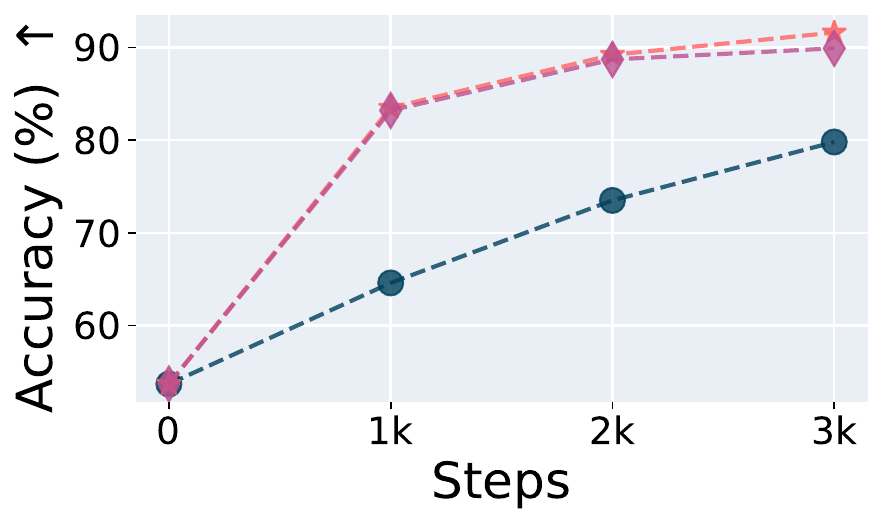} & 
        \includegraphics[width=0.3\columnwidth]{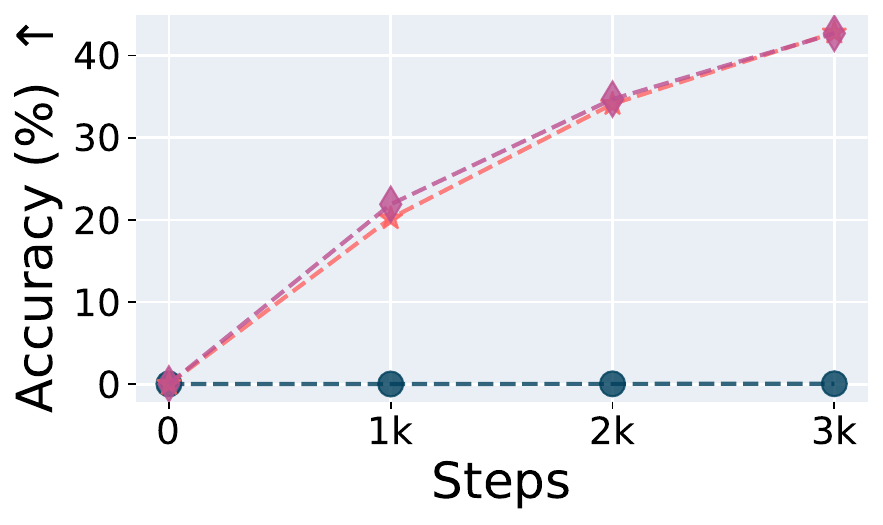} \\
        \footnotesize{(a) OPT-1.3B on SST-2} & 
        \footnotesize{(b) ViT-B on CIFAR-100} \\
    \end{tabular}
    \caption{Performance of ZO-Muon with different projection matrix resampling interval $v$, demonstrated by (a) fine-tuning OPT-1.3B on SST-2 and (b) fine-tuning ViT-B on CIFAR-100. }
    \label{fig:sample-interval}
\end{figure}
\label{appendix:ablation-study}

\paragraph{Projection matrix resampling interval}
Recall that in ZO-Muon we \emph{lazily} resample the subspace projection matrix $\mathbf{P}$ every 100 steps (denoted by the resampling interval $v=100$). To validate this design choice, \textbf{Fig.\,\ref{fig:sample-interval}} compares ZO-Muon with different resampling intervals $v \in \{1, 100, 500\}$. As shown, overly frequent resampling does not improve performance on either LLM or ViT fine-tuning. The case of $v=1$ consistently underperforms $v=100$, while $v=500$ achieves performance comparable to $v=100$. This is because frequent subspace switching (\textit{i.e.}, resampling $\mathbf{P}$) prevents the optimizer from sufficiently exploiting gradient information within a given subspace, leading to less stable training. Our lazy resampling strategy is also aligned with previous works on both FO training \citep{zhao2024galore,refael2025sumo} and ZO training \citep{chen2025enhancing,yu2025zeroth}

\paragraph{Projection matrix sampling strategies.} We present additional studies on sampling strategies for the projection matrices used in Subspace-MeZO and ZO-Muon, with results summarized in \textbf{Tab.~\ref{tab:projection-matrix-sampling}}. Our default strategy, denoted as ``\textit{Random},'' constructs the projection matrix \(\mathbf{P} \in \mathbb{R}^{m \times r}\) via the QR decomposition of a Gaussian matrix. Alternatively, the ``\textit{Sketching}'' strategy applies Gaussian sketching \citep{halko2011finding} to past descending directions to capture informative subspace directions, a technique also utilized in first-order subspace training \citep{he2025low,refael2025adarankgrad}. Specifically, we maintain a momentum term \(\mathbf{M} \in \mathbb{R}^{m \times n}\) representing the moving average of past ZO gradients. The matrix \(\mathbf{P}\) is then derived by performing QR decomposition on \(\mathbf{M}\mathbf{Q}\), where \(\mathbf{Q} \in \mathbb{R}^{n \times r}\) is drawn from a Gaussian distribution. As shown in Tab.~\ref{tab:projection-matrix-sampling}, Gaussian sketching does not yield notable performance gains over random sampling for either Subspace-MeZO or ZO-Muon, yet it increases GPU memory consumption, particularly for LLMs, due to the requirement of storing the momentum \(\mathbf{M}\).
In the ViT fine-tuning scenario, Gaussian sketching exhibits even worse accuracy, likely due to the cumulative bias introduced by Momentum as it accumulates past ZO gradient estimates rather than preciser FO gradients.

\begin{table}[htb]
\centering
\small
\caption{Performance (evaluated by accuracy and F1) and GPU memory (short for ``Mem'') of Subspace-MeZO and ZO-Muon with different projection matrix sampling strategies.}
\label{tab:projection-matrix-sampling}
\resizebox{\columnwidth}{!}{%
    \begin{threeparttable}
        \begin{tabular}{llcccccccccc}
        \toprule
        \toprule
        \multirow{3}{*}{\textbf{Method}} & \multirow{3}{*}{\shortstack[l]{\textbf{Sampling} \\ \textbf{Strategy}}} & \multicolumn{6}{c}{\cellcolor{lightblue}\textbf{LlaMA3-8B Fine-tuning}} & \multicolumn{4}{c}{\cellcolor{lightgreen}\textbf{ViT-B Fine-tuning}} \\
        \cmidrule(lr){3-8} \cmidrule(lr){9-12}
         & & \multicolumn{2}{c}{\textbf{SST-2}} & \multicolumn{2}{c}{\textbf{RTE}} & \multicolumn{2}{c}{\textbf{SQuAD}} & \multicolumn{2}{c}{\textbf{CIFAR-10}} & \multicolumn{2}{c}{\textbf{CIFAR-100}} \\
        \cmidrule(lr){3-4} \cmidrule(lr){5-6} \cmidrule(lr){7-8} \cmidrule(lr){9-10} \cmidrule(lr){11-12}
         & & Acc. (\%) $\uparrow$ & Mem (MiB) $\downarrow$ & Acc. (\%) $\uparrow$ & Mem (MiB) $\downarrow$ & F1 (\%) $\uparrow$ & Mem (MiB) $\downarrow$ & Acc. (\%) $\uparrow$ & Mem (MiB) $\downarrow$ & Acc. (\%) $\uparrow$ & Mem (MiB) $\downarrow$ \\
        \midrule
        \multirow{2}{*}{Subspace-MeZO} & Random & 92.3 & 19058 & 68.6 & 28242 & 84.5 & 32102 & 94.2 & 1634 & 58.2 & 4086 \\ 
         & Sketching & 92.3 & 34814 & 67.5 & 38874 & 87.2 & 44734 & 81.2 & 1898 & 29.1 & 4088 \\ 
        \midrule
        \multirow{2}{*}{ZO-Muon} & Random & 94.3 & 19176 & 81.2 & 27224 & 88.2 & 30718 & 94.9 & 1686 & 72.6 & 4410 \\
         & Sketching & 93.8 & 35268 & 78.3 & 44304 & 88.3 & 45248 & 95.3 & 1950 & 59.8 & 4412 \\  
        \bottomrule
        \bottomrule
        \end{tabular}%
    \end{threeparttable}
}
\vspace*{-1mm}
\end{table}

    
    

\begin{figure}[htb]
    \centering
    \begin{tabular}{cc}
        \includegraphics[width=0.3\columnwidth]{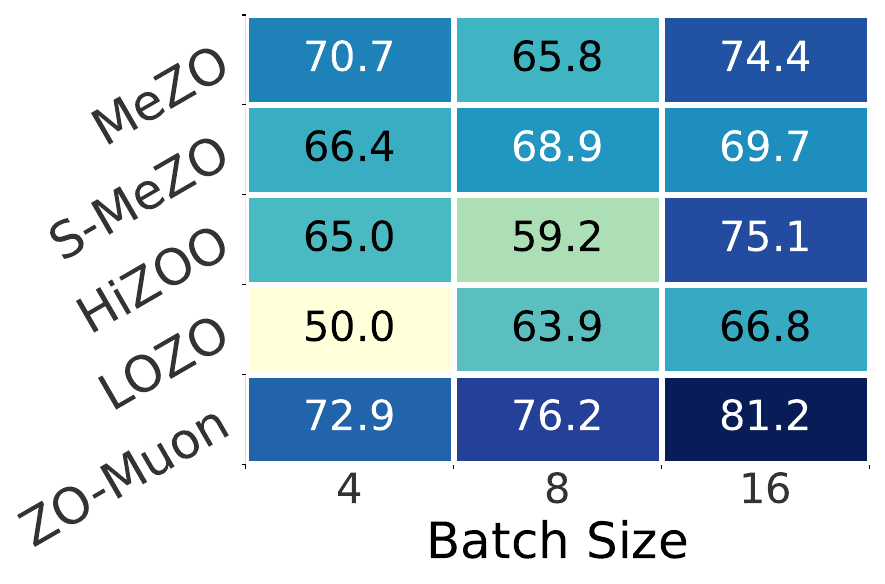} & 
        \includegraphics[width=0.3\columnwidth]{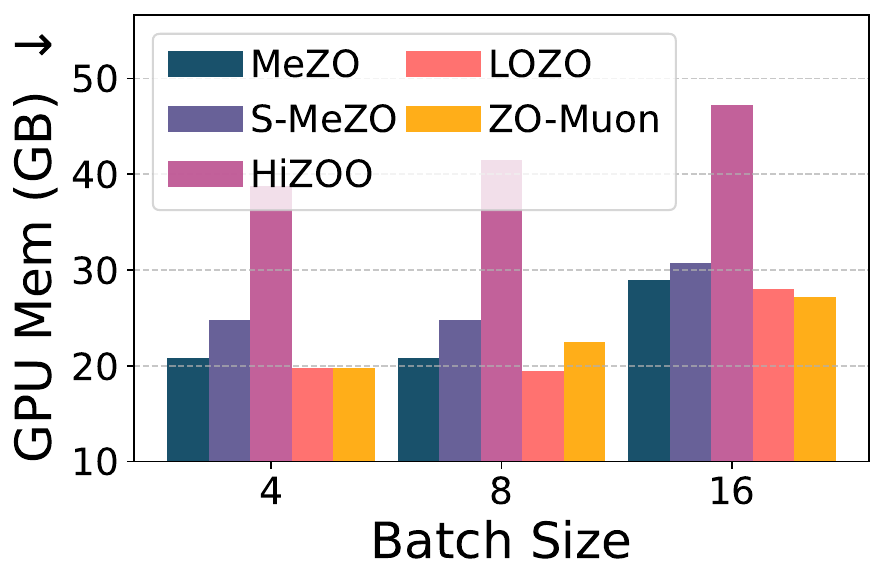} \\
        \footnotesize{(a) Fine-tuning accuracy} & 
        \footnotesize{(b) GPU memory} \\
    \vspace*{-2mm}
    \end{tabular}  
    \caption{Fine-tuning accuracy and GPU memory comparison under varying batch sizes. ZO-Muon versus baselines on LlaMA3-8B fine-tuning for RTE with the same number of iterations.}
    \label{fig:llm-bs}
\end{figure}

\paragraph{ZO-Muon does not rely on large batch sizes.} We extend the evaluation of ZO-Muon with decreased batch sizes in Fig.\,\ref{fig:vit-bs} to LLMs, evaluated by fine-tuning LlaMA3-8B on RTE, shown in \textbf{Fig.\,\ref{fig:llm-bs}}. The results are aligned with the ViT results: ZO-Muon outperforms the baselines by a large margin with a decreased batch size. And ZO-Muon with a batch size of 4 outperforms S-MeZO and LOZO with a batch size of 16.

\section{Limitations}
\label{appendix:limitations}
Despite the extensive experiments and the clear performance/efficiency gain of ZO-Muon, our work has limitations that highlight future directions. First, we focus on memory-efficient fine-tuning, and extending ZO-Muon to full pretraining of large-scale models remains open. Second, exploring the application of ZO optimization in other post-training tasks, such as reinforcement learning and safety alignment, remains open.

\section{Ethics Statement}
\label{appendix:ethics-statement}
The research presented in this paper conforms to the NeurIPS Code of Ethics. Our research focuses on memory-efficient training of large models. This research does not involve human subjects, crowdsourcing, or the release of new models or datasets that carry a high risk of misuse. Because our contribution is primarily algorithmic and methodological, we do not foresee direct negative societal impacts arising specifically from this work.

\section{LLM Usage}
\label{appendix:llm-usage}
In this paper, the sole purpose of LLMs is to assist with improving the fluency of the paper, such
as refining the grammar. At no point did the language model contribute to research ideas or to the
generation of original content.

\newpage
\section*{NeurIPS Paper Checklist}

\begin{enumerate}

\item {\bf Claims}
    \item[] Question: Do the main claims made in the abstract and introduction accurately reflect the paper's contributions and scope?
    \item[] Answer: \answerYes{} 
    \item[] Justification:  The main claims in the abstract and introduction accurately reflect the paper’s contributions and scope.
    \item[] Guidelines:
    \begin{itemize}
        \item The answer \answerNA{} means that the abstract and introduction do not include the claims made in the paper.
        \item The abstract and/or introduction should clearly state the claims made, including the contributions made in the paper and important assumptions and limitations. A \answerNo{} or \answerNA{} answer to this question will not be perceived well by the reviewers. 
        \item The claims made should match theoretical and experimental results, and reflect how much the results can be expected to generalize to other settings. 
        \item It is fine to include aspirational goals as motivation as long as it is clear that these goals are not attained by the paper. 
    \end{itemize}

\item {\bf Limitations}
    \item[] Question: Does the paper discuss the limitations of the work performed by the authors?
    \item[] Answer: \answerYes{} 
    \item[] Justification: Limitations are discussed in Appendix~\ref{appendix:limitations}.
    \item[] Guidelines:
    \begin{itemize}
        \item The answer \answerNA{} means that the paper has no limitation while the answer \answerNo{} means that the paper has limitations, but those are not discussed in the paper. 
        \item The authors are encouraged to create a separate ``Limitations'' section in their paper.
        \item The paper should point out any strong assumptions and how robust the results are to violations of these assumptions (e.g., independence assumptions, noiseless settings, model well-specification, asymptotic approximations only holding locally). The authors should reflect on how these assumptions might be violated in practice and what the implications would be.
        \item The authors should reflect on the scope of the claims made, e.g., if the approach was only tested on a few datasets or with a few runs. In general, empirical results often depend on implicit assumptions, which should be articulated.
        \item The authors should reflect on the factors that influence the performance of the approach. For example, a facial recognition algorithm may perform poorly when image resolution is low or images are taken in low lighting. Or a speech-to-text system might not be used reliably to provide closed captions for online lectures because it fails to handle technical jargon.
        \item The authors should discuss the computational efficiency of the proposed algorithms and how they scale with dataset size.
        \item If applicable, the authors should discuss possible limitations of their approach to address problems of privacy and fairness.
        \item While the authors might fear that complete honesty about limitations might be used by reviewers as grounds for rejection, a worse outcome might be that reviewers discover limitations that aren't acknowledged in the paper. The authors should use their best judgment and recognize that individual actions in favor of transparency play an important role in developing norms that preserve the integrity of the community. Reviewers will be specifically instructed to not penalize honesty concerning limitations.
    \end{itemize}

\item {\bf Theory assumptions and proofs}
    \item[] Question: For each theoretical result, does the paper provide the full set of assumptions and a complete (and correct) proof?
    \item[] Answer: \answerYes{} 
    \item[] Justification: Theoretical proofs are detailed in Appendix~\ref{appendix:convergence}.
    \item[] Guidelines:
    \begin{itemize}
        \item The answer \answerNA{} means that the paper does not include theoretical results. 
        \item All the theorems, formulas, and proofs in the paper should be numbered and cross-referenced.
        \item All assumptions should be clearly stated or referenced in the statement of any theorems.
        \item The proofs can either appear in the main paper or the supplemental material, but if they appear in the supplemental material, the authors are encouraged to provide a short proof sketch to provide intuition. 
        \item Inversely, any informal proof provided in the core of the paper should be complemented by formal proofs provided in appendix or supplemental material.
        \item Theorems and Lemmas that the proof relies upon should be properly referenced. 
    \end{itemize}

    \item {\bf Experimental result reproducibility}
    \item[] Question: Does the paper fully disclose all the information needed to reproduce the main experimental results of the paper to the extent that it affects the main claims and/or conclusions of the paper (regardless of whether the code and data are provided or not)?
    \item[] Answer: \answerYes{} 
    \item[] Justification: Detailed setups are reported in Appendix~\ref{appendix:experiment-details}.
    \item[] Guidelines:
    \begin{itemize}
        \item The answer \answerNA{} means that the paper does not include experiments.
        \item If the paper includes experiments, a \answerNo{} answer to this question will not be perceived well by the reviewers: Making the paper reproducible is important, regardless of whether the code and data are provided or not.
        \item If the contribution is a dataset and\slash or model, the authors should describe the steps taken to make their results reproducible or verifiable. 
        \item Depending on the contribution, reproducibility can be accomplished in various ways. For example, if the contribution is a novel architecture, describing the architecture fully might suffice, or if the contribution is a specific model and empirical evaluation, it may be necessary to either make it possible for others to replicate the model with the same dataset, or provide access to the model. In general. releasing code and data is often one good way to accomplish this, but reproducibility can also be provided via detailed instructions for how to replicate the results, access to a hosted model (e.g., in the case of a large language model), releasing of a model checkpoint, or other means that are appropriate to the research performed.
        \item While NeurIPS does not require releasing code, the conference does require all submissions to provide some reasonable avenue for reproducibility, which may depend on the nature of the contribution. For example
        \begin{enumerate}
            \item If the contribution is primarily a new algorithm, the paper should make it clear how to reproduce that algorithm.
            \item If the contribution is primarily a new model architecture, the paper should describe the architecture clearly and fully.
            \item If the contribution is a new model (e.g., a large language model), then there should either be a way to access this model for reproducing the results or a way to reproduce the model (e.g., with an open-source dataset or instructions for how to construct the dataset).
            \item We recognize that reproducibility may be tricky in some cases, in which case authors are welcome to describe the particular way they provide for reproducibility. In the case of closed-source models, it may be that access to the model is limited in some way (e.g., to registered users), but it should be possible for other researchers to have some path to reproducing or verifying the results.
        \end{enumerate}
    \end{itemize}

\item {\bf Open access to data and code}
    \item[] Question: Does the paper provide open access to the data and code, with sufficient instructions to faithfully reproduce the main experimental results, as described in supplemental material?
    \item[] Answer: \answerYes{} 
    \item[] Justification: The data used in this paper are fully open datasets. The code and instructions
are included in the supplementary material.
    \item[] Guidelines:
    \begin{itemize}
        \item The answer \answerNA{} means that paper does not include experiments requiring code.
        \item Please see the NeurIPS code and data submission guidelines (\url{https://neurips.cc/public/guides/CodeSubmissionPolicy}) for more details.
        \item While we encourage the release of code and data, we understand that this might not be possible, so \answerNo{} is an acceptable answer. Papers cannot be rejected simply for not including code, unless this is central to the contribution (e.g., for a new open-source benchmark).
        \item The instructions should contain the exact command and environment needed to run to reproduce the results. See the NeurIPS code and data submission guidelines (\url{https://neurips.cc/public/guides/CodeSubmissionPolicy}) for more details.
        \item The authors should provide instructions on data access and preparation, including how to access the raw data, preprocessed data, intermediate data, and generated data, etc.
        \item The authors should provide scripts to reproduce all experimental results for the new proposed method and baselines. If only a subset of experiments are reproducible, they should state which ones are omitted from the script and why.
        \item At submission time, to preserve anonymity, the authors should release anonymized versions (if applicable).
        \item Providing as much information as possible in supplemental material (appended to the paper) is recommended, but including URLs to data and code is permitted.
    \end{itemize}

\item {\bf Experimental setting/details}
    \item[] Question: Does the paper specify all the training and test details (e.g., data splits, hyperparameters, how they were chosen, type of optimizer) necessary to understand the results?
    \item[] Answer: \answerYes{} 
    \item[] Justification: Detailed experiment setup is reported in Appendix~\ref{appendix:experiment-details}.
    \item[] Guidelines:
    \begin{itemize}
        \item The answer \answerNA{} means that the paper does not include experiments.
        \item The experimental setting should be presented in the core of the paper to a level of detail that is necessary to appreciate the results and make sense of them.
        \item The full details can be provided either with the code, in appendix, or as supplemental material.
    \end{itemize}

\item {\bf Experiment statistical significance}
    \item[] Question: Does the paper report error bars suitably and correctly defined or other appropriate information about the statistical significance of the experiments?
    \item[] Answer: \answerNo{}{} 
    \item[] Justification: Given the large volume of quantitative results, the paper does not explicitly report error bars for each data point.
    \item[] Guidelines:
    \begin{itemize}
        \item The answer \answerNA{} means that the paper does not include experiments.
        \item The authors should answer \answerYes{} if the results are accompanied by error bars, confidence intervals, or statistical significance tests, at least for the experiments that support the main claims of the paper.
        \item The factors of variability that the error bars are capturing should be clearly stated (for example, train/test split, initialization, random drawing of some parameter, or overall run with given experimental conditions).
        \item The method for calculating the error bars should be explained (closed form formula, call to a library function, bootstrap, etc.)
        \item The assumptions made should be given (e.g., Normally distributed errors).
        \item It should be clear whether the error bar is the standard deviation or the standard error of the mean.
        \item It is OK to report 1-sigma error bars, but one should state it. The authors should preferably report a 2-sigma error bar than state that they have a 96\% CI, if the hypothesis of Normality of errors is not verified.
        \item For asymmetric distributions, the authors should be careful not to show in tables or figures symmetric error bars that would yield results that are out of range (e.g., negative error rates).
        \item If error bars are reported in tables or plots, the authors should explain in the text how they were calculated and reference the corresponding figures or tables in the text.
    \end{itemize}

\item {\bf Experiments compute resources}
    \item[] Question: For each experiment, does the paper provide sufficient information on the computer resources (type of compute workers, memory, time of execution) needed to reproduce the experiments?
    \item[] Answer: \answerYes{} 
    \item[] Justification: Detailed compute resources are reported in Appendix~\ref{appendix:experiment-details}.
    \item[] Guidelines:
    \begin{itemize}
        \item The answer \answerNA{} means that the paper does not include experiments.
        \item The paper should indicate the type of compute workers CPU or GPU, internal cluster, or cloud provider, including relevant memory and storage.
        \item The paper should provide the amount of compute required for each of the individual experimental runs as well as estimate the total compute. 
        \item The paper should disclose whether the full research project required more compute than the experiments reported in the paper (e.g., preliminary or failed experiments that didn't make it into the paper). 
    \end{itemize}
    
\item {\bf Code of ethics}
    \item[] Question: Does the research conducted in the paper conform, in every respect, with the NeurIPS Code of Ethics \url{https://neurips.cc/public/EthicsGuidelines}?
    \item[] Answer: \answerYes{} 
    \item[] Justification: The authors confirm that in every respect, the research is not against NeurIPS
code of Ethics.
    \item[] Guidelines:
    \begin{itemize}
        \item The answer \answerNA{} means that the authors have not reviewed the NeurIPS Code of Ethics.
        \item If the authors answer \answerNo, they should explain the special circumstances that require a deviation from the Code of Ethics.
        \item The authors should make sure to preserve anonymity (e.g., if there is a special consideration due to laws or regulations in their jurisdiction).
    \end{itemize}

\item {\bf Broader impacts}
    \item[] Question: Does the paper discuss both potential positive societal impacts and negative societal impacts of the work performed?
    \item[] Answer: \answerNA{} 
    \item[] Justification: We are proposing an algorithm for efficient training, and do not have such concerns. See Appendix~\ref{appendix:ethics-statement} for details.
    \item[] Guidelines:
    \begin{itemize}
        \item The answer \answerNA{} means that there is no societal impact of the work performed.
        \item If the authors answer \answerNA{} or \answerNo, they should explain why their work has no societal impact or why the paper does not address societal impact.
        \item Examples of negative societal impacts include potential malicious or unintended uses (e.g., disinformation, generating fake profiles, surveillance), fairness considerations (e.g., deployment of technologies that could make decisions that unfairly impact specific groups), privacy considerations, and security considerations.
        \item The conference expects that many papers will be foundational research and not tied to particular applications, let alone deployments. However, if there is a direct path to any negative applications, the authors should point it out. For example, it is legitimate to point out that an improvement in the quality of generative models could be used to generate Deepfakes for disinformation. On the other hand, it is not needed to point out that a generic algorithm for optimizing neural networks could enable people to train models that generate Deepfakes faster.
        \item The authors should consider possible harms that could arise when the technology is being used as intended and functioning correctly, harms that could arise when the technology is being used as intended but gives incorrect results, and harms following from (intentional or unintentional) misuse of the technology.
        \item If there are negative societal impacts, the authors could also discuss possible mitigation strategies (e.g., gated release of models, providing defenses in addition to attacks, mechanisms for monitoring misuse, mechanisms to monitor how a system learns from feedback over time, improving the efficiency and accessibility of ML).
    \end{itemize}
    
\item {\bf Safeguards}
    \item[] Question: Does the paper describe safeguards that have been put in place for responsible release of data or models that have a high risk for misuse (e.g., pre-trained language models, image generators, or scraped datasets)?
    \item[] Answer: \answerNA{} 
    \item[] Justification: We do not have such safety concerns.
    \item[] Guidelines:
    \begin{itemize}
        \item The answer \answerNA{} means that the paper poses no such risks.
        \item Released models that have a high risk for misuse or dual-use should be released with necessary safeguards to allow for controlled use of the model, for example by requiring that users adhere to usage guidelines or restrictions to access the model or implementing safety filters. 
        \item Datasets that have been scraped from the Internet could pose safety risks. The authors should describe how they avoided releasing unsafe images.
        \item We recognize that providing effective safeguards is challenging, and many papers do not require this, but we encourage authors to take this into account and make a best faith effort.
    \end{itemize}

\item {\bf Licenses for existing assets}
    \item[] Question: Are the creators or original owners of assets (e.g., code, data, models), used in the paper, properly credited and are the license and terms of use explicitly mentioned and properly respected?
    \item[] Answer: \answerYes{}{} 
    \item[] Justification: All the open datasets and models are properly cited.
    \item[] Guidelines:
    \begin{itemize}
        \item The answer \answerNA{} means that the paper does not use existing assets.
        \item The authors should cite the original paper that produced the code package or dataset.
        \item The authors should state which version of the asset is used and, if possible, include a URL.
        \item The name of the license (e.g., CC-BY 4.0) should be included for each asset.
        \item For scraped data from a particular source (e.g., website), the copyright and terms of service of that source should be provided.
        \item If assets are released, the license, copyright information, and terms of use in the package should be provided. For popular datasets, \url{paperswithcode.com/datasets} has curated licenses for some datasets. Their licensing guide can help determine the license of a dataset.
        \item For existing datasets that are re-packaged, both the original license and the license of the derived asset (if it has changed) should be provided.
        \item If this information is not available online, the authors are encouraged to reach out to the asset's creators.
    \end{itemize}

\item {\bf New assets}
    \item[] Question: Are new assets introduced in the paper well documented and is the documentation provided alongside the assets?
    \item[] Answer: \answerNA{} 
    \item[] Justification: The paper does not release new assets.
    \item[] Guidelines:
    \begin{itemize}
        \item The answer \answerNA{} means that the paper does not release new assets.
        \item Researchers should communicate the details of the dataset\slash code\slash model as part of their submissions via structured templates. This includes details about training, license, limitations, etc. 
        \item The paper should discuss whether and how consent was obtained from people whose asset is used.
        \item At submission time, remember to anonymize your assets (if applicable). You can either create an anonymized URL or include an anonymized zip file.
    \end{itemize}

\item {\bf Crowdsourcing and research with human subjects}
    \item[] Question: For crowdsourcing experiments and research with human subjects, does the paper include the full text of instructions given to participants and screenshots, if applicable, as well as details about compensation (if any)? 
    \item[] Answer: \answerNA{} 
    \item[] Justification: The paper does not involve crowdsourcing nor research with human subjects.
    \item[] Guidelines:
    \begin{itemize}
        \item The answer \answerNA{} means that the paper does not involve crowdsourcing nor research with human subjects.
        \item Including this information in the supplemental material is fine, but if the main contribution of the paper involves human subjects, then as much detail as possible should be included in the main paper. 
        \item According to the NeurIPS Code of Ethics, workers involved in data collection, curation, or other labor should be paid at least the minimum wage in the country of the data collector. 
    \end{itemize}

\item {\bf Institutional review board (IRB) approvals or equivalent for research with human subjects}
    \item[] Question: Does the paper describe potential risks incurred by study participants, whether such risks were disclosed to the subjects, and whether Institutional Review Board (IRB) approvals (or an equivalent approval/review based on the requirements of your country or institution) were obtained?
    \item[] Answer: \answerNA{} 
    \item[] Justification: Our research does not concern human subjects.
    \item[] Guidelines: The paper does not involve crowdsourcing nor research with human subjects.
    \begin{itemize}
        \item The answer \answerNA{} means that the paper does not involve crowdsourcing nor research with human subjects.
        \item Depending on the country in which research is conducted, IRB approval (or equivalent) may be required for any human subjects research. If you obtained IRB approval, you should clearly state this in the paper. 
        \item We recognize that the procedures for this may vary significantly between institutions and locations, and we expect authors to adhere to the NeurIPS Code of Ethics and the guidelines for their institution. 
        \item For initial submissions, do not include any information that would break anonymity (if applicable), such as the institution conducting the review.
    \end{itemize}

\item {\bf Declaration of LLM usage}
    \item[] Question: Does the paper describe the usage of LLMs if it is an important, original, or non-standard component of the core methods in this research? Note that if the LLM is used only for writing, editing, or formatting purposes and does \emph{not} impact the core methodology, scientific rigor, or originality of the research, declaration is not required.
    \item[] Answer: \answerNA{} 
    \item[] Justification: The core method development in this research does not involve LLMs as any important, original, or non-standard components. See Appendix~\ref{appendix:llm-usage} for details.
    \item[] Guidelines:
    \begin{itemize}
        \item The answer \answerNA{} means that the core method development in this research does not involve LLMs as any important, original, or non-standard components.
        \item Please refer to our LLM policy in the NeurIPS handbook for what should or should not be described.
    \end{itemize}

\end{enumerate}


\end{document}